\documentclass[lettersize,journal]{IEEEtran}
\usepackage{amsmath,amsfonts}
\usepackage{algorithmic}
\usepackage{algorithm}
\usepackage{array}
\usepackage[caption=false,font=normalsize,labelfont=sf,textfont=sf]{subfig}
\usepackage{textcomp}
\usepackage{stfloats}
\usepackage{url}
\usepackage{verbatim}
\usepackage{graphicx}
\usepackage{cite}
\usepackage{xcolor}
\usepackage{colortbl}
\usepackage{amsmath}
\usepackage{amssymb}
\usepackage{mathtools}
\usepackage{amsthm}
\usepackage{hyperref}
\usepackage{cleveref}
\usepackage{makecell}
\usepackage{multirow,multicol}
\usepackage{pifont}
\usepackage{booktabs} 

\usepackage[normalem]{ulem}
\useunder{\uline}{\ul}{}
\usepackage{xcolor}
\usepackage[switch]{lineno}

\theoremstyle{plain}
\newtheorem{theorem}{Theorem}[section]
\newtheorem{proposition}[theorem]{Proposition}

\theoremstyle{definition}

\theoremstyle{remark}
\newtheorem{remark}[theorem]{Remark}
\usepackage[square,numbers,sort&compress]{natbib}

\definecolor{myblue}{RGB}{0, 102, 204} 
\definecolor{darkyellow}{RGB}{255,204,0}

\newcommand{\logopic}{%
   \raisebox{-0.5ex}{
      \includegraphics[height=2.5ex]{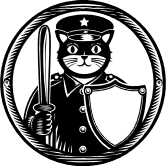}
  }%
}

\begin{document}
\title{\logopic EraseAnything++: Enabling Concept Erasure in Rectified Flow Transformers Leveraging Multi-Object Optimization}

\author{
  Zhaoxin Fan,
  Nanxiang Jiang,
  Daiheng Gao,
  Shiji Zhou\thanks{
    Zhaoxin Fan, Nanxiang Jiang, Shiji Zhou, Wenjun Wu are with Beijing Advanced Innovation Center for Future Blockchain and Privacy Computing, School of Artificial Intelligence, Beihang University. 
    \par Daiheng Gao is with University of Science and Technology of China. 
    \par Shiji Zhou is the corresponding author (Email: zhoushiji25@buaa.edu.cn).
  },
  Wenjun Wu
}

\markboth{Journal of \LaTeX\ Class Files,~Vol.~14, No.~8, August~2021}%
{Shell \MakeLowercase{\textit{et al.}}: A Sample Article Using IEEEtran.cls for IEEE Journals}


\maketitle

\begin{abstract}

Removing undesired concepts from large-scale text-to-image (T2I) and text-to-video (T2V) diffusion models—while preserving overall generative quality—remains a major challenge, particularly as modern models such as Stable Diffusion v3, Flux, and OpenSora employ flow-matching and transformer-based architectures, as well as extending to long-horizon video generation. Existing concept erasure methods, designed for earlier T2I/T2V models, often fail to generalize to these paradigms. To address this, we propose EraseAnything++, a unified framework for concept erasure in both image and video diffusion models with flow-matching objectives. Central to our approach is formulating concept erasure as a constrained multi-objective optimization problem, explicitly balancing concept removal with the preservation of generative utility. To solve such problem with conflicting objectives, we introduce an efficient utility-preserving unlearning strategy based on implicit gradient surgery. Furthermore, by integrating LoRA-based parameter tuning with attention-level regularization, our method anchors erasure on key visual representations and propagates it consistently across spatial and temporal dimensions. In the video setting, we further enhance consistency through an anchor-and-propagate mechanism, which initializes erasure on reference frames and enforces it throughout subsequent transformer layers, thereby mitigating temporal drift. Extensive experiments on both image and video benchmarks demonstrate that EraseAnything++ substantially outperforms prior methods in erasure effectiveness, generative fidelity, and temporal consistency, establishing a new state of the art for concept erasure in next-generation diffusion models. Code is available at \url{https://github.com/nxjiang-jnx/EraseAnything-PlusPlus}.

\end{abstract}


\section{Introduction}
\label{intro:basic}
Since the emergence of DALL-E 2~\cite{dalle2} and Stable Diffusion (SD)\cite{sd}, text-to-image (T2I) and text-to-video (T2V) diffusion models have driven remarkable progress in generative vision. Recently, models such as Flux~\cite{flux}  and OpenSora~\cite{zheng2024open} have adopted flow-matching~\cite{flowmatching, liu2022flow} and transformer-based~\cite{attention} architectures to deliver significant gains in prompt following, image fidelity, and output diversity. These advances have extended the generative paradigm from static images to long-horizon videos and introduced sophisticated modules, including advanced text encoders like T5~\cite{T5} and rotary positional encodings (RoPE)~\cite{rope}, fundamentally reshaping the landscape of T2I and T2V modeling.

\begin{figure}[t]
\centering
\includegraphics[width=0.8\linewidth]{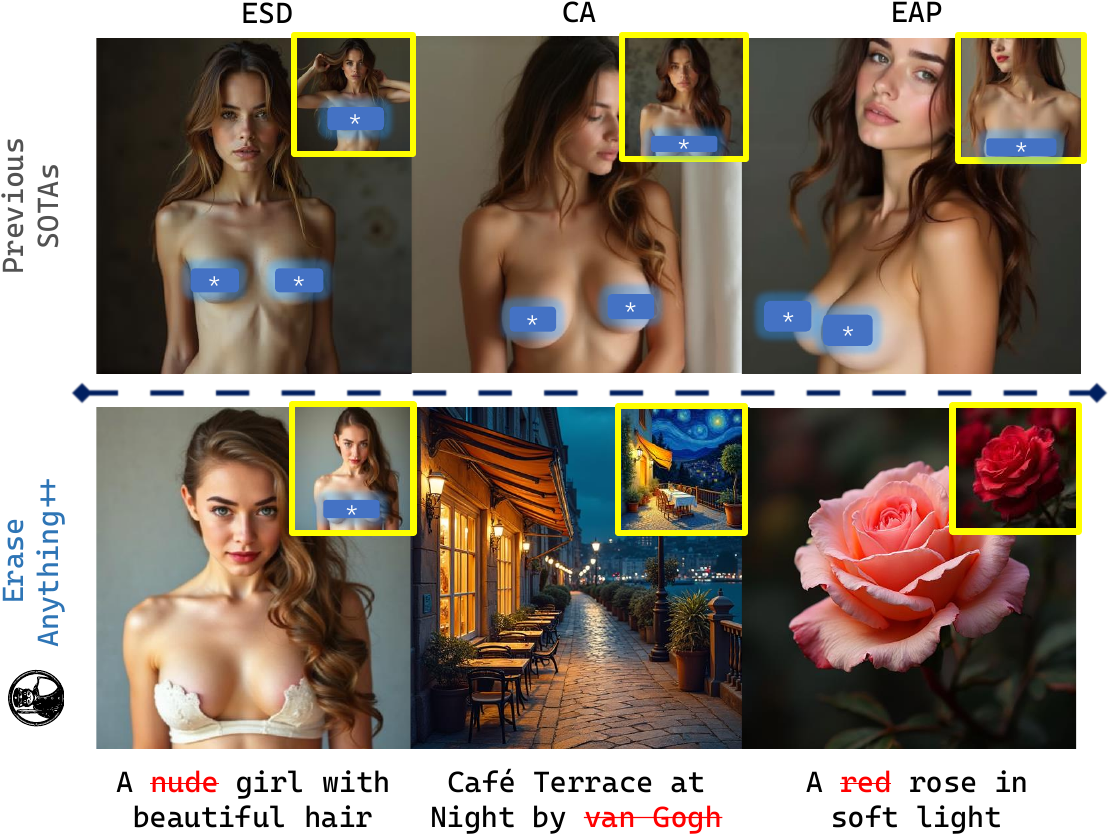}
\vspace{-0.1in}
\caption{\textbf{Illustration of Difference between EraseAnything++ and Existing Methods}. In this paper, we introduce {EraseAnything++}, an advanced concept erasure technique for Flow-based T2I/T2V images. \textit{First row}: Classical concept-erasing methods—ESD~\cite{esd}, CA~\cite{ca}, and EAP~\cite{eap}—have been transplanted into Flux.1 [dev] and are tested with the input ``$\mathtt{nudity}$" (\textcolor{myblue}{blue} bars indicate author-added sensory harmony). \textit{Second row}: Visualizing EraseAnything++'s impact—pre and post-concept removal. Original outputs (\textcolor{darkyellow}{yellow} box) are displayed in the upper right.}
\label{fig:teaser}
\vspace{-0.25in}
\end{figure}

Despite these successes, the increasing scale and diversity of training data have amplified concerns regarding the generation of inappropriate or unsafe content. In particular, diffusion models can produce NSFW (Not Safe For Work) material in response to certain prompts, a risk that is widely acknowledged in both academic and public discourse~\cite{barez2025openproblemsmachineunlearning, times2024deepfakes, BBC2025, test_the_boundary}. Accordingly, concept erasure~\cite{ca, esd}—the targeted suppression or removal of specific concepts—has become a critical safety requirement for the responsible deployment of T2I and T2V systems.

To address these risks, a variety of concept erasure techniques have been proposed, primarily in the context of earlier architectures such as SD, which rely on DDPM/DDIM~\cite{ddpm, ddim} and U-Net~\cite{unet} frameworks. These methods can be grouped into three major categories: (1) direct fine-tuning approaches (\textit{e.g.}, ESD~\cite{esd}), which suppress activations linked to specific concepts; (2) closed-form solutions (\textit{e.g.}, UCE~\cite{uce}), which analytically remove concept-related features; and (3) adversarial training strategies (e.g., EAP~\cite{eap}), which leverage negative samples or prompts for unlearning. While these approaches have proven effective in static image settings, they encounter notable limitations in the context of modern flow-matching, transformer-based models such as Flux. Importantly, these techniques not only struggle to generalize across architectures but also lack the capacity to effectively address concept erasure in video generation. In video diffusion models, concepts persist and interact through complex spatio-temporal attention, leading to phenomena such as temporal drift and error accumulation—issues that existing CE methods are fundamentally unequipped to resolve. Although methods such as T2VUnlearning \cite{t2vunlearning} and VideoEraser \cite{videoeraser} have been proposed for concept erasure in video diffusion models, they still struggle to balance the removal of undesired content with the preservation of relevant information. Specifically, T2VUnlearning often exhibits an over-erasure behavior: while the target concept is suppressed, the main subject may also disappear or be severely degraded, accompanied by noticeable visual artifacts. 
VideoEraser, on the other hand, tends to reduce temporal dynamics, producing videos with weak motion or near-static frames in which the action described by the prompt is never carried out, which compromises the naturalness and temporal fidelity of the generated results. As a result, these approaches are prone to introducing visual artifacts and often fail to maintain effective erasure or consistency in longer video sequences.

To bridge this gap, we introduce \textbf{EraseAnything++}, a comprehensive framework for concept erasure in both image and video diffusion models built on flow-matching objectives. Specifically, we first propose a unified multi-objective optimization (MOO) framework for both image and video settings, where we design a performance-preserving gradient surgery method with theoretical support to balance concept removal and content preservation. In more detail, at the image level, we employ LoRA-based~\cite{lora} parameter adaptation to suppress model activations associated with undesirable concepts, while explicitly preserving performance on a balanced set of irrelevant concepts. This is achieved through targeted fine-tuning using an ESD-inspired~\cite{esd} loss, an attention map regularizer tailored to the internal structure of modern models, and a novel self-contrastive loss~\cite{contrastive1, contrastive2} that penalizes the model for confusing erased concepts with irrelevant ones, leveraging both negative samples and diverse distractors. For the video setting, EraseAnything++ further extends these ideas to address the unique challenges of temporal consistency and error propagation. We propose an anchor-and-propagate strategy: erasure is first anchored on selected reference frames, and then consistently propagated through the model’s spatio-temporal transformer layers. By incorporating performance-preserving gradient surgery to ensure that erasure updates do not unduly compromise generative fidelity or introduce temporal artifacts, our approach enables stable, controllable suppression of target concepts across long video sequences, effectively mitigating the shortcomings of prior work.

We validate EraseAnything++ on extensive image and video benchmarks. Experimental results demonstrate that our method substantially outperforms existing techniques in erasure effectiveness, generative fidelity, and temporal consistency. These findings establish EraseAnything++ as a robust and general solution for concept erasure in the latest generation of diffusion models.

Notably, \textbf{this work builds upon our earlier ICML conference version, EraseAnything~\cite{gao2025eraseanything}, and introduces several key extensions in EraseAnything++}: 1) We provide a formal treatment of concept erasure by defining it as a multi-objective optimization problem, extending our previous optimization method~\cite{eupmu} and offering a rigorous mathematical model for the inherent trade-off between concept removal and generative fidelity. 2) Based on this theoretical framework, we propose a new optimization strategy employing efficient gradient projection, enabling more stable and controllable unlearning during training. 3) We extend the original EraseAnything approach from image-based concept erasure to video, addressing the unique challenges of spatio-temporal consistency and error propagation in text-to-video diffusion models. 4) We significantly expand our experimental evaluation, conducting extensive experiments on a broader set of benchmarks—including several dedicated video concept erasure datasets—to rigorously assess the performance and generality of EraseAnything++.

\begin{itemize}
\item  We present EraseAnything++, a unified and generalizable framework for concept erasure in both text-to-image and text-to-video diffusion models, supporting modern architectures based on flow matching and transformers.

\item   We formulate concept erasure as a multi-objective optimization problem, providing a rigorous theoretical model that balances concept removal with the preservation of generative fidelity.

\item  We design new optimization strategies, including LoRA-based parameter adaptation, attention-map regularization, self-contrastive loss, and an efficient gradient projection mechanism, enabling stable and controllable erasure in both images and videos.

\item  We conduct extensive experiments across a wide range of benchmarks—including dedicated video concept erasure datasets—demonstrating that EraseAnything++ achieves state-of-the-art effectiveness, fidelity, and temporal consistency in both image and video generation tasks.
\end{itemize}
\section{Related Work}

\subsection{Generative Diffusion Models for Images and Videos}

The landscape of generative vision has witnessed a paradigm shift, evolving from early Generative Adversarial Networks (GANs)~\cite{GAN} to the now dominant Diffusion Probabilistic Models (DPMs)~\cite{ddpm}. Initial breakthroughs in text-to-image (T2I) generation are marked by models such as GLIDE~\cite{nichol2021glide}, DALL-E 2~\cite{dalle2}, and Imagen~\cite{imagen}, which demonstrate unprecedented semantic control. The release of the Stable Diffusion (SD) series~\cite{sd, sdxl} further democratize this technology, leveraging latent diffusion within a U-Net~\cite{unet} architecture to balance computational efficiency and generation quality.Recently, the field has gravitated towards \textit{Diffusion Transformers} (DiTs) combined with Flow Matching objectives, moving away from the traditional U-Net and epsilon-prediction paradigms. Stable Diffusion 3 (SD3)~\cite{sd3} exemplifies this shift, treating the forward noising process as a Rectified Flow~\cite{liu2022flow} to establish a straight path between data and noise distributions. SD3 employs a Multimodal Diffusion Transformer (MMDiT) architecture, where text and pixel modalities are processed as sequences of embeddings. Specifically, positional encodings are applied to flattened $2\times2$ latent patches, which are then fused with text embeddings from a trio of encoders ($\mathtt{CLIP L/14, OpenCLIP bigG/14, T5 \, XXL}$)~\cite{clip, T5} before passing through modulated attention blocks. Building upon this foundation, Flux~\cite{flux} has emerged as a state-of-the-art contender. By scaling the flow-matching transformer architecture to 12B parameters and incorporating rotary positional encodings (RoPE)~\cite{rope}, Flux achieves superior performance in prompt adherence, typography, and visual fidelity. This generative paradigm has naturally extended to Text-to-Video (T2V) synthesis. While early T2V models adapted U-Net architectures by inserting temporal attention layers~\cite{video_diffusion_models, ho2022imagenvideo, singer2022makeavideo}, the latest generation follows the success of DiT in images. Models like CogVideo~\cite{cogvideo}, HunyuanVideo~\cite{hunyuanvideo}, Open-Sora~\cite{opensora, opensora2.0} adopt a spatio-temporal patchification strategy, treating video as a long sequence of tokens. Open-Sora, in particular, democratizes efficient video production by employing a Video DiT architecture trained with flow-matching objectives. It supports variable resolutions and durations by effectively modeling long-range temporal dependencies through space-time attention mechanisms.

In this work, we focus on this new wave of Flow-Matching Transformer models. We select Flux and Open-Sora as our primary experimental testbeds for image and video concept erasure, respectively. Their shared architectural principles allow us to propose a unified unlearning framework that addresses the unique challenges of these next-generation generative models.

\subsection{Concept Erasure}

Large-scale datasets like LAION-5B~\cite{laion5b} empower T2I and T2V models but introduce risks of generating inappropriate or copyrighted content. To mitigate this, early approaches focus on training datasets filtering~\cite{sd}, and post-generation content filtering~\cite{rando2022red}. For example, SD 2 filters training data, but this is costly and significantly affects model performance. Alternatively, libraries like Diffusers~\cite{diffusers} use post-hoc safety checkers. However, these are easily bypassed by users and do not address the root cause. Consequently, research has shifted toward concept erasure via model fine-tuning. Methods such as ESD~\cite{esd}, UCE~\cite{uce}, MACE~\cite{lu2024mace}, and SPM~\cite{lyu2024one} modify model weights to suppress specific concepts. Recent work places greater emphasis on preserving irrelevant concepts to maintain general capability. For instance, EAP~\cite{eap} employs adversarial learning to retain unrelated semantics, while Real-Era~\cite{liu2024realera} addresses ``concept residue" by regularizing associated concepts to boost specificity. As diffusion models expand to the video domain, erasure becomes more complex due to temporal dynamics. Recent attempts like T2VUnlearning~\cite{t2vunlearning} and VideoEraser~\cite{videoeraser} address erasure in T2V models. However, these methods often struggle to balance concept removal with the preservation of relevant information. Consequently, they are prone to introducing visual artifacts and often fail to maintain consistency across longer video sequences.

In this work, we aim to bridge these gaps in both image and video domains, with a specific focus on modern Flow-based Transformer models. A key challenge here is the text encoder. Unlike previous methods that rely on CLIP's word-level embeddings, Flux and Open-Sora use T5. T5's sentence-level embeddings make standard similarity metrics less effective for identifying irrelevant concepts. To address this, we propose a heuristic approach using Large Language Models (LLMs) to dynamically select irrelevant concepts. We elaborate on this T5-specific challenge and our solution in Section \ref{sec:sec3}.

\subsection{Multi-Objective Optimization}

Multi-Objective Optimization (MOO) provides a theoretical framework for optimizing conflicting tasks simultaneously. Early approaches primarily relied on dynamic loss re-weighting strategies. These methods adjust task weights based on homoscedastic uncertainty~\cite{kendall2018multi}, gradient magnitudes (\textit{e.g.}, GradNorm~\cite{chen2018gradnorm}), or the relative complexity of training tasks~\cite{guo2018dynamic}. To address gradient interference more directly, \cite{sener2018multi} formulates Multi-Task Learning (MTL) as an MOO problem, proposing the Multiple Gradient Descent Algorithm (MGDA) to find a Pareto stationary point. Subsequent works have introduced various techniques to resolve gradient conflicts. PCGrad~\cite{yu2020gradient} projects conflicting gradients onto the normal plane of others to prevent destructive interference. Similarly, GradDrop~\cite{chen2020just} stochastically discards conflicting gradient components based on their sign, while RotoGrad~\cite{javaloy2021rotograd} aligns gradients through rotation. CAGrad~\cite{liu2021conflict} takes a different approach by constraining the update direction to remain within a specific region around the average gradient.

Recent research has focused on improving the computational efficiency and theoretical guarantees of MOO solvers~\cite{MO-MIX}. For instance, MoCo~\cite{fernando2023mitigating} extends MGDA to a probabilistic setting with convergence analysis, and FAMO~\cite{liu2023famo} significantly reduces the computational overhead associated with gradient calculations. However, these general-purpose MOO methods are not directly applicable to concept erasure. In the context of erasure, the objectives are fundamentally asymmetric—requiring the precise removal of specific concepts while broadly preserving the model's remaining generative utility. Standard MOO strategies often fail to balance this delicate trade-off, leading to either incomplete erasure or catastrophic forgetting of unrelated concepts.

\section{Unified Optimization Framework}
\label{sec:method_framework}

In this section, we present the theoretical foundation of EraseAnything++. We formulate concept erasure as a constrained multi-objective optimization problem. Our goal is to maximize the erasure of a target concept while strictly bounding the degradation of unrelated concepts (preservation). To solve this efficiently, we propose a gradient projection method and a fast approximation algorithm.
We clarify that the constrained MOO formulation and the implicit gradient-surgery solver are built upon the utility-preserving unlearning framework in our previous work~\cite{eupmu}. Our contribution in this work is to instantiate and extend this principle for concept erasure in rectified-flow Transformer generators. Specifically, we define erasure and preservation objectives tailored to flow-based T2I/T2V models, and turn the general utility-preserving unlearning principle into a practical framework for modern image and video diffusion Transformers.

\subsection{Problem Formulation}
Consider a diffusion model parameterized by $\boldsymbol{\theta}$. Let $\mathcal{L}_{e}(\boldsymbol{\theta})$ denote the \textit{erasure objective} (minimizing the likelihood of the target concept) and $\mathcal{L}_{p}(\boldsymbol{\theta})$ denote the \textit{preservation objective} (maintaining the likelihood of irrelevant concepts).

At iteration $t$, we update the parameters via $\boldsymbol{\theta}_{t+1} = \boldsymbol{\theta}_t - \alpha_t \boldsymbol{d}_t$, where $\alpha_t$ is the step size and $\boldsymbol{d}_t$ is the update direction. We define the improvement for each objective as:
\begin{equation}\label{eq:relative_changes}
\begin{aligned}
& r_p(\alpha_t, \boldsymbol{d}_t) = \mathcal{L}_{p}(\boldsymbol{\theta}_t) - \mathcal{L}_{p}(\boldsymbol{\theta}_{t+1}), \\ 
& r_e(\alpha_t, \boldsymbol{d}_t) = \mathcal{L}_{e}(\boldsymbol{\theta}_t) - \mathcal{L}_{e}(\boldsymbol{\theta}_{t+1}).
\end{aligned}
\end{equation}
To achieve safe erasure, we seek a direction $\boldsymbol{d}_t$ that maximizes the erasure improvement $r_e$ while ensuring the preservation degradation $r_p$ remains within a controlled tolerance $\varepsilon_t \geq 0$. Mathematically, we pose this as:
\begin{equation}\label{eq:optimization_problem}
\begin{aligned}
\max_{\boldsymbol{d}_t} &\quad \frac{1}{\alpha_t} r_e(\alpha_t, \boldsymbol{d}_t) - \frac{1}{2} \left\|\boldsymbol{d}_t\right\|^2 \\
\text{s.t.} &\quad \frac{1}{\alpha_t} r_p(\alpha_t, \boldsymbol{d}_t) \geq - \varepsilon_t,
\end{aligned}
\end{equation}
where $\left\|\boldsymbol{d}_t\right\|^2$ acts as regularization to prevent unbounded updates. The constraint ensures that any increase in the preservation loss is bounded by $\varepsilon_t$.

\subsection{Explicit Unilateral Gradient Surgery}
Since the step size $\alpha_t$ is typically small, we apply a first-order Taylor approximation to Eq.~\eqref{eq:relative_changes}:
\begin{equation}
r_p(\alpha_t, \boldsymbol{d}_t) \approx \alpha_t \nabla \mathcal{L}_p (\boldsymbol{\theta}_t) \cdot \boldsymbol{d}_t, \quad 
r_e(\alpha_t, \boldsymbol{d}_t) \approx \alpha_t \nabla \mathcal{L}_e (\boldsymbol{\theta}_t) \cdot \boldsymbol{d}_t.
\end{equation}
Consequently, Problem~\eqref{eq:optimization_problem} is approximated by:
\begin{equation}\label{eq:approximated_optimization_problem}
\begin{aligned}
\max_{\boldsymbol{d}_t} & \quad \nabla \mathcal{L}_e(\boldsymbol{\theta}_t) \cdot \boldsymbol{d}_t - \frac{1}{2} \left\|\boldsymbol{d}_t\right\|^2 \\
\text{s.t.} & \quad \nabla \mathcal{L}_p(\boldsymbol{\theta}_t) \cdot \boldsymbol{d}_t \geq - \varepsilon_t.
\end{aligned}
\end{equation}
This formulation aims to align $\boldsymbol{d}_t$ with the erasure gradient $\nabla \mathcal{L}_e$ while strictly constraining its projection onto the preservation gradient $\nabla \mathcal{L}_p$.

\begin{proposition}[Dual Problem] \label{prop:dual}
    The dual objective of Problem~\eqref{eq:approximated_optimization_problem} is:
    \begin{equation}\label{eq:dual_approximated_optimization_problem}
    \min_{\lambda_t\geq 0} L_t(\lambda_t) = \frac{1}{2}\left\|\nabla \mathcal{L}_e(\boldsymbol{\theta}_t) + \lambda_t \nabla \mathcal{L}_p (\boldsymbol{\theta}_t) \right\|^2 + \lambda_t \varepsilon_t.
    \end{equation}
\end{proposition}

This quadratic programming problem yields a closed-form analytical solution for the optimal update direction $\boldsymbol{d}_t^*$.

\begin{proposition}[Closed-Form Solution] \label{prop:closed_form_dual}
    The optimal direction $\boldsymbol{d}_t^*$ is given by:
    \begin{equation}\label{eq:solve_approximated_optimization_problem}
    \boldsymbol{d}_t^* = \begin{cases}
    \nabla \mathcal{L}_{e}(\boldsymbol{\theta}_t) + \lambda_t^* \nabla \mathcal{L}_{p}(\boldsymbol{\theta}_t), & \text{if } \lambda_t^* > 0 \\
    \nabla \mathcal{L}_{e}(\boldsymbol{\theta}_t), & \text{if } \lambda_t^* \leq 0
    \end{cases}
    \end{equation}
    where the optimal Lagrange multiplier $\lambda_t^*$ is:
    \begin{equation}\label{eq:lambda}
        \lambda_t^* = \frac{-\nabla \mathcal{L}_{p}(\boldsymbol{\theta}_t) \cdot \nabla \mathcal{L}_{e}(\boldsymbol{\theta}_t) - \varepsilon_t}{\left\|\nabla \mathcal{L}_{p}(\boldsymbol{\theta}_t)\right\|^2}.
    \end{equation}
\end{proposition}
We defer derivations to Appendix A-B and A-C.

\begin{figure}[t]
    \centering
    \includegraphics[width=0.6\linewidth]{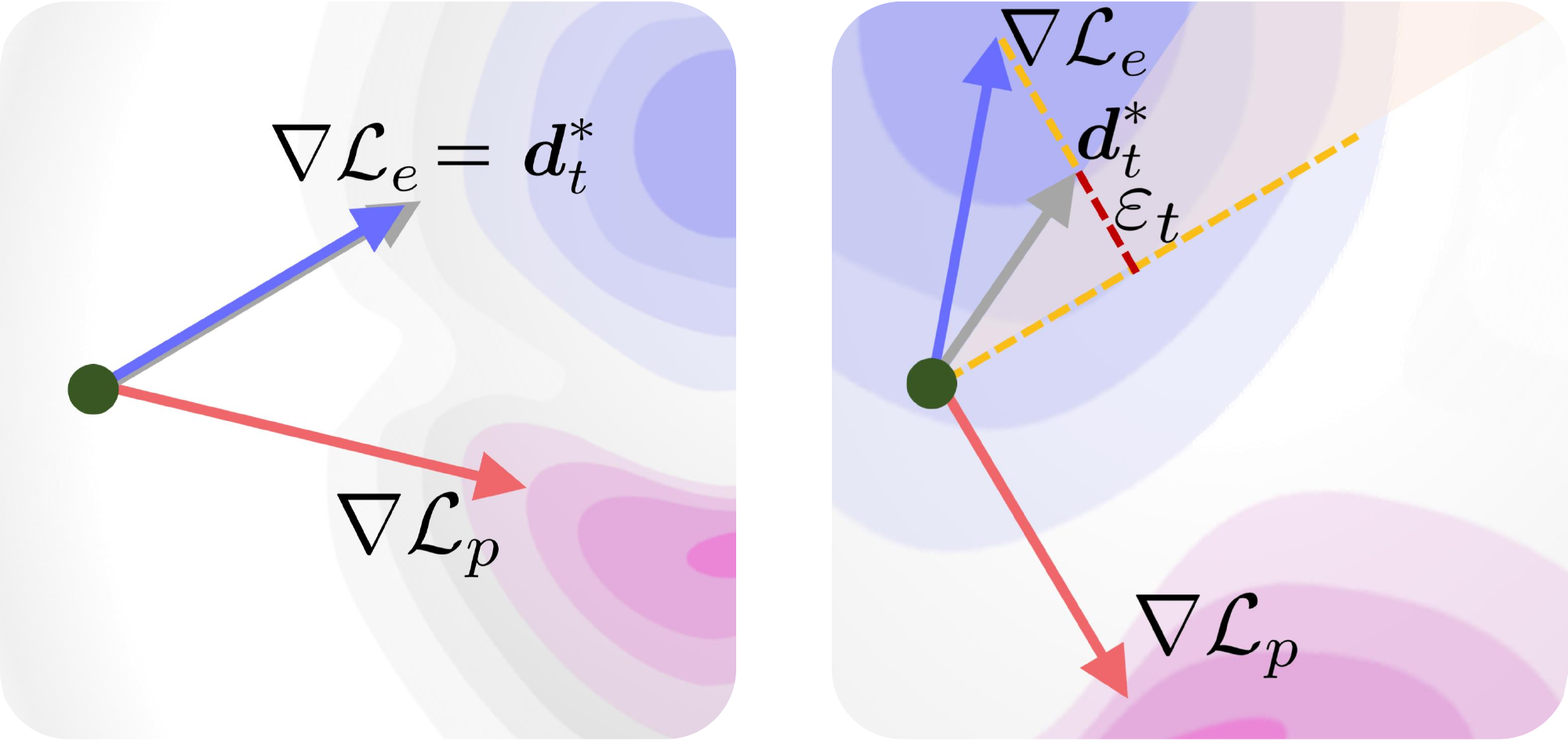}
    \vspace{-0.1in}
    \caption{\textbf{Geometric interpretation of the proposed gradient surgery.} \textbf{(Left)} When the erasure gradient $\nabla \mathcal{L}_e$ lies within the safe region (non-conflicting or satisfying the tolerance), the update direction $\boldsymbol{d}_t^*$ remains unchanged. \textbf{(Right)} When a significant conflict arises, $\boldsymbol{d}_t^*$ is obtained by projecting $\nabla \mathcal{L}_e$ to satisfy the preservation constraint. The tolerance $\varepsilon_t$ creates a ``tolerance cone'' (dashed yellow line) that relaxes strict orthogonality, preventing optimization deadlocks while ensuring controlled utility preservation.}
    \label{fig:gradient_geometry}
    \vspace{-0.2in}
\end{figure}

\textbf{Geometric Interpretation.} 
The tolerance $\varepsilon_t$ plays a pivotal role in navigating the trade-off between erasure and preservation. 
As shown in Fig.~\ref{fig:gradient_geometry} (Left), when the erasure gradient $\nabla \mathcal{L}_e$ naturally satisfies the preservation constraint, no modification is needed. 
However, strictly enforcing non-degradation ($\varepsilon_t = 0$) often leads to optimization stagnation when gradients are diametrically opposed. 
By introducing $\varepsilon_t > 0$, we effectively construct a ``tolerance cone'' as illustrated in Fig.~\ref{fig:gradient_geometry} (Right). 
Gradient surgery is triggered only when the conflict exceeds this threshold (i.e., $\nabla \mathcal{L}_p \cdot \nabla \mathcal{L}_e < -\varepsilon_t$). 
This relaxation allows the erasure process to proceed even under moderate gradient conflict, preventing the ``deadlock'' phenomenon common in strict multi-objective optimization.

\begin{remark}\label{remark:utility_preservation}
    By imposing local constraints at each step, we can bound the global utility degradation:
    \begin{equation}\label{eq:decline}
            \mathcal{L}_p(\boldsymbol{\theta}_t) - \mathcal{L}_p(\boldsymbol{\theta}_0) \lesssim \mathcal{O}\left(\sum_{i=1}^t\varepsilon_t \alpha_t\right).
    \end{equation}
    Proof details are in Appendix A-D. This guarantees that preservation performance is controllable via hyperparameters.
\end{remark}

\subsection{Implicit Efficient Gradient Surgery}

While the explicit solution guarantees utility preservation, it requires computing both $\nabla \mathcal{L}_e$ and $\nabla \mathcal{L}_p$ separately. This doubles the computational cost compared to standard fine-tuning, undermining the efficiency required for practical unlearning. To address this, we propose an \textit{Implicit Efficient Gradient Surgery} that approximates the solution with the cost of a single backpropagation.

We solve for the weighting factor $\lambda_t$ via gradient descent approximation rather than exact computation. The update rule for $\lambda_t$ is:
\begin{equation*}
\lambda_{t+1} = \lambda_t - \beta_t \nabla_{\lambda_t} L_t(\lambda_t).
\end{equation*}
Computing $\nabla_{\lambda_t} L_t$ directly still requires separate gradients. Instead, we use a first-order approximation:
\begin{equation*}
\begin{aligned}
\nabla_{\lambda_t} L_t(\lambda_t)
&= \nabla \mathcal{L}_{p}(\boldsymbol{\theta}_t)\cdot ( \nabla \mathcal{L}_e(\boldsymbol{\theta}_t) + \lambda_t \nabla \mathcal{L}_p (\boldsymbol{\theta}_t)) + \varepsilon_t\\
&= \nabla \mathcal{L}_{p}(\boldsymbol{\theta}_t)\cdot \boldsymbol{d}_t + \varepsilon_t\\
&\approx \frac{1}{\alpha_t} ( \mathcal{L}_{p}(\boldsymbol{\theta}_t) - \mathcal{L}_{p}(\boldsymbol{\theta}_{t+1})) + \varepsilon_t.
\end{aligned}
\end{equation*}
This yields a gradient-free update for $\lambda$:
\begin{equation}\label{eq:lambda_update_approximate}
\lambda_{t+1} = \lambda_t - \beta_t \tilde{\delta}_t, \quad \text{where } \tilde{\delta}_t = \frac{1}{\alpha_t} (\mathcal{L}_{p}(\boldsymbol{\theta}_t) - \mathcal{L}_{p}(\boldsymbol{\theta}_{t+1})) + \varepsilon_t.
\end{equation}

\begin{algorithm}[t]
   \caption{Training Procedure of EraseAnything++}
   \label{alg:algorithm}
\begin{algorithmic}
   \STATE {\bfseries Input:} Pretrained model $\boldsymbol{\theta}_0$, erasure concept set $\mathcal{D}_{e}$, preservation concept set $\mathcal{D}_{p}$, total steps $M$.
   \STATE {\bfseries Hyperparameters:} Learning rates $\alpha$ (model), $\beta$ (dual variable), tolerance $\varepsilon$, initial $\lambda_0 = 0$.
   
   \FOR{iteration $t=0$ {\bfseries to} $M-1$}
       \STATE \textbf{\ding{182}} Sample batches $\mathcal{B}_{e} \sim \mathcal{D}_{e}$ and $\mathcal{B}_{p} \sim \mathcal{D}_{p}$.
       \STATE \textbf{\ding{183}} Estimate preservation loss drift (via Eq.~\ref{eq:lambda_update_approximate}):
       \STATE \quad $\tilde{\delta}_t = \frac{1}{\alpha} (\mathcal{L}_{p}(\boldsymbol{\theta}_{t}) - \mathcal{L}_{p}(\boldsymbol{\theta}_{t+1})) + \varepsilon$.
       \STATE \textbf{\ding{184}} Update dual variable (implicit surgery):
       \STATE \quad $\lambda_{t+1} \leftarrow \max(0, \lambda_t - \beta \tilde{\delta}_t)$.
       \STATE \textbf{\ding{185}} Compute composite objective:
       \STATE \quad $\mathcal{L}_{total} = \mathcal{L}_{e}(\boldsymbol{\theta}_t) + \lambda_{t+1} \mathcal{L}_{p}(\boldsymbol{\theta}_t)$.
       \STATE \textbf{\ding{186}} Update model parameters:
       \STATE \quad $\boldsymbol{\theta}_{t+1} \leftarrow \boldsymbol{\theta}_t - \alpha \nabla_{\boldsymbol{\theta}} \mathcal{L}_{total}$.
   \ENDFOR
   \STATE {\bfseries Output:} Unlearned model parameters $\boldsymbol{\theta}_M$.
\end{algorithmic}

\end{algorithm}

To implement the proposed approximation method, we adopt a streamlined optimization procedure, outlined in Algorithm~\ref{alg:algorithm}, that bypasses the dual computational burden of standard gradient surgery. Specifically, the process begins by updating the dynamic weight $\lambda_t$ based on forward-pass loss changes (Eq.~\eqref{eq:lambda_update_approximate}) without requiring backpropagation. Subsequently, the final update direction $\boldsymbol{d}_t$ is derived via a \textit{single} backpropagation of the composite objective $\mathcal{L}_{e} + \lambda_t \mathcal{L}_{p}$, followed by standard parameter updates. By eliminating the necessity to explicitly compute and store separate gradient vectors for each objective, this strategy reduces computational overhead, ensuring scalability for large-scale diffusion models.

\subsection{Theoretical Analysis}
We provide theoretical guarantees for the convergence and optimality of our implicit approximation.

\begin{theorem}[Approximate $\lambda^*$]\label{thm:approximate_lambda}
Assume $\mathcal{L}_p$ and $\mathcal{L}_e$ are $G$-Smooth and $L$-Lipschitz. With appropriate step sizes satisfying $\sum \alpha_i \leq \mathcal{O}(1)$ and $\sum \varepsilon_i \leq \mathcal{O}(1)$, the average gap between our approximate $\lambda_t$ and the optimal $\lambda_t^*$ is bounded by:
    \begin{equation}
       \frac{1}{t} \sum_{i=1}^t \left(L_i(\lambda_i) - L_i(\lambda_i^*)\right) \leq \mathcal{O}(1/t^{1/3}).
    \end{equation}
\end{theorem}
This confirms that as training progresses, $\lambda_t$ converges to the optimal weight that ensures utility preservation.

\begin{theorem}[Pareto Optimality]\label{thm:Pareto_Optimality}
Under convex assumptions, there exists a composite loss $\mathcal{C}(\boldsymbol{\theta})$ such that our algorithm converges to a Pareto optimal solution with rate:
\begin{equation}
    \mathcal{C}(\boldsymbol{\theta}_t) - \min _{\boldsymbol{\theta} }\mathcal{C}(\boldsymbol{\theta}) \leq \mathcal{O}(1/t).
\end{equation}
\end{theorem}
This matches the convergence rate of state-of-the-art first-order MOO algorithms. 

\begin{remark}\label{remark:unlearning_opt}
    Combining Eq.~\eqref{eq:decline} and Theorem~\ref{thm:Pareto_Optimality}, we show that our method reaches the optimal erasure solution $\boldsymbol{\theta}^* = \max_{\boldsymbol{\theta}} \mathcal{L}_e(\boldsymbol{\theta})$ subject to the strictly bounded preservation constraint.
\end{remark}

\begin{theorem}[Pareto Stationary]\label{thm:Pareto_Stationary}
In non-convex scenarios (general deep learning), the algorithm converges to a Pareto stationary point:
\begin{equation}
    \min_{i=1,\ldots,t} \min_{(\mu_e,\mu_p)\in \Delta_2} \left\|\mu_e \nabla \mathcal{L}_e(\boldsymbol{\theta}_i) + \mu_p \nabla \mathcal{L}_p(\boldsymbol{\theta}_i) \right\| \leq \mathcal{O}(1/t^{1/2}).
\end{equation}
\end{theorem}
This ensures effective optimization even in the complex non-convex landscapes of diffusion models. Detailed proofs are provided in Appendix A-E to A-G.
\section{Obstacles in Migrating Concept Erasure to Modern Transformers}
\label{sec:sec3}

In this section, we analyze why classical erasure methods developed for Stable Diffusion (SD) fail when applied to modern architectures like Flux and Open-Sora. We identify three primary barriers: the shift from U-Net to Transformer backbones, the distinct characteristics of the T5 text encoder, and the introduction of temporal dimensions in video generation. These structural differences render direct adaptation of traditional methods infeasible.

\textbf{Architectural Mismatch:} 
The first hurdle lies in the fundamental difference between the U-Net architecture of SD and the Transformer-based architecture of Flux and Open-Sora. Classical methods such as ESD~\cite{esd}, UCE~\cite{uce}, and MACE~\cite{lu2024mace} primarily target explicit cross-attention layers to suppress concept activation. However, when adapting these methods to modern flow-based transformers (\textit{e.g.}, Flux) and video diffusion transformers (\textit{e.g.}, Open-Sora), we encounter an important challenge: explicit cross attention layer does not exist in either dual stream blocks or single stream blocks. The absence of isolated cross-attention layers means that traditional weights manipulation cannot be directly transplanted. This structural gap leads to ``concept residue", where target concepts are only partially removed, necessitating a redesigned erasure approach.

\textbf{The T5 Text Encoder Challenge:} 
A significant obstacle shared by both Flux and Open-Sora is their reliance on the T5 text encoder. While SD utilizes CLIP as its standard text encoder for image guidance, Flux and Open-Sora rely on T5. Unlike CLIP,  which is optimized for word-level alignment, T5 is designed for sentence-level understanding. As shown in \cref{table:table_1}, we extract the T5 feature for the word ``nude" and compared it with the entire vocabulary (over 30,000 words) from the T5 default tokenizer. The cosine similarity reveals the top 3 closest synonyms based on semantic embeddings. However, these results were far from rational, indicating that T5's word-level embeddings are not reliable for this task and cannot serve as an effective evaluator of semantic similarity.

Another significant issue lies in the size of the T5 embeddings. With a shape of $\mathtt{max\_sequence\_length(256, 4096)}$, T5 embeddings are approximately \textbf{18 times larger} than CLIP embeddings, which have a shape of $\mathtt{(77, 768)}$. Consequently, this increased dimensionality makes the adaptive selection of adversarial prompts—a core component of methods like EAP~\cite{eap}—computationally prohibitive.

\textbf{Unified Attention Analysis:} 
Despite these architectural changes, we hypothesize that concept-specific activations must still exist within the network. Inspired by previous studies~\cite{hertz2022prompt, xie2023boxdiff}, we examined the internal features of Flux. As shown in Fig.~\ref{fig:mot}, even without explicit cross-attention, a linear correlation exists between text embeddings and the intermediate attention maps. 

Specifically, the attention weights $\mathbf{W_{attn}}$ are computed by concatenating textual and pixel embeddings:
\begin{equation}
\begin{aligned}
& \mathbf{Q} = \mathtt{concat}(\mathbf{Q}_{text}, \mathbf{Q}_{pixel}, \mathtt{dim}=-1), \\
& \mathbf{K} = \mathtt{concat}(\mathbf{K}_{text}, \mathbf{K}_{pixel}, \mathtt{dim}=-1), \\
& \mathbf{W_{attn}} = \mathtt{Softmax}(\mathbf{Q} \times \mathbf{K}).
\end{aligned}
\label{eq:eq1}
\end{equation}
This formulation suggests that the nexus between text and image is inherently forged within $\mathbf{W_{attn}}$. By identifying the token index of a target word, we can theoretically nullify its influence by zeroing out the corresponding columns in the attention map.

\definecolor{lightgray}{gray}{0.9}
\begin{table}[t]
\vspace{-0.1in}
\caption{Comparison of synonym discovery for \textbf{nude}.}
\vspace{-0.25in}
\label{table:table_1}
\vskip 0.15in
\begin{center}
\begin{small}
\begin{sc}
\begin{tabular}{lcr}
\toprule
Method & Top-3 closest synonyms \\
\midrule
Claude 3.5   & "naked", "undressed", "unclothed" \\
GPT-4o & "bare", "naked", "unclothed" \\
Kimi    & "naked", "unclothed", "bare" \\
\rowcolor{blue!10}
T5 feature   & 'lean', 'deer', 'girl'     \\
\bottomrule
\end{tabular}
\end{sc}
\end{small}
\end{center}
\vspace{-0.2in}
\end{table}

\begin{figure}[t]
\centering
\includegraphics[width=0.8\linewidth]{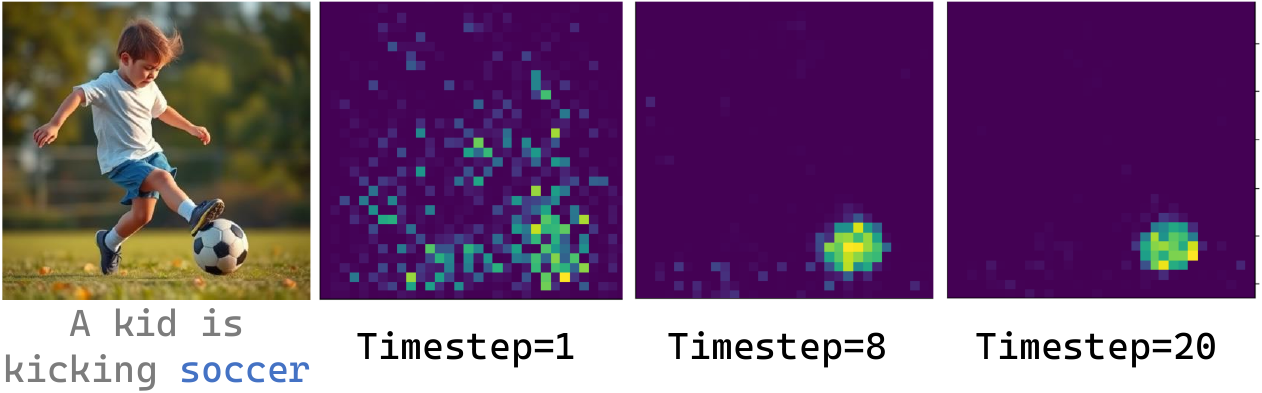}
\vspace{-0.1in}
\caption{\textbf{Correlations between text and attention maps in Flux.} Despite the lack of explicit cross-attention layers, a linear relationship persists between text embeddings and the unified attention map.}
\label{fig:mot}
\vspace{-0.25in}
\end{figure}

\textbf{Vulnerability to Obfuscation:} 
While manually zeroing attention columns (as shown in Fig.~\ref{fig:mot2}) can erase a concept, this naive approach is fragile. Our experiments reveal it is highly susceptible to prompt attacks. Simple obfuscations—such as altering the prompt with nonsensical suffixes (\textcolor[rgb]{0.1, 0.2, 0.4}{\textbf{soccer}} $\rightarrow$ \textcolor{red}{\textbf{soccerrs}}) or intentional misspellings (\textcolor[rgb]{0.1, 0.2, 0.4}{\textbf{Nike}} $\rightarrow$ \textcolor{red}{\textbf{Nikke}})—can bypass this filter. In these cases, the token mapping shifts, rendering the targeted erasure futile while the model still successfully generates the concept.

\textbf{Temporal Dynamics in Video:} 
Migrating to video models like Open-Sora introduces a new dimension of complexity: temporal propagation. Unlike static images, video generation relies on temporal attention layers to maintain consistency across frames. A concept erased in the initial frame may inadvertently reappear in subsequent frames due to information leakage through temporal attention mechanisms. Consequently, image-based erasure methods often fail to prevent ``concept drift" over time. This necessitates a method that not only erases concepts spatially but also enforces suppression consistently across the temporal axis, motivating our proposed EraseAnything++.

\begin{figure}[t]
\centering
\vspace{-0.1in}
\includegraphics[width=0.85\linewidth]{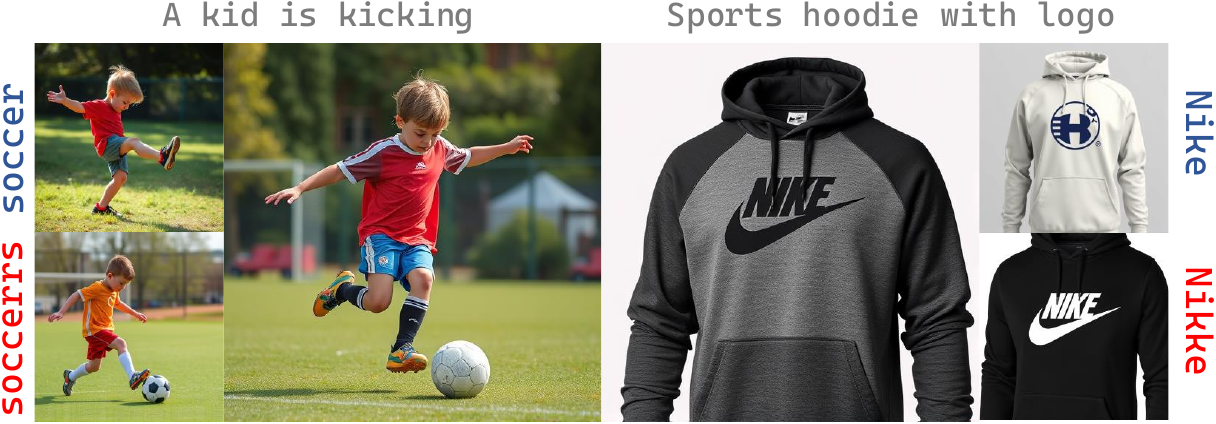}
\vspace{-0.1in}
\caption{\textbf{Limitations of Attention Map erasure.} Zeroing out attention columns ($\mathbf{W_{attn}}[:,:,idx_{i}] = 0$) works for simple prompts but fails against adversarial attacks like keyword obfuscation (\textit{e.g.}, "soccerrs").}
\label{fig:mot2}
\vspace{-0.25in}
\end{figure}
\section{Method}
\label{sec:method}

Building upon the Multi-Objective Optimization (MOO) framework established in Section~\ref{sec:method_framework}, we now detail the specific objective functions for EraseAnything++. Our method operates on a unified principle across both image (Flux) and video (Open-Sora) modalities: we decompose the unlearning task into two distinct objectives—\textbf{Erasure} ($\mathcal{L}_e$) and \textbf{Preservation} ($\mathcal{L}_p$). In the image setting, we employ a modified ESD loss combined with attention regularization to define $\mathcal{L}_e$, while leveraging LoRA-based retention and a novel Reverse Self-Contrastive (RSC) loss for $\mathcal{L}_p$. In the video setting, we extend these definitions via an \textit{Anchor-and-Propagate} strategy, addressing the unique challenge of temporal consistency in 3D Video DiTs. Finally, these objectives are dynamically balanced within our established optimization framework. We now introduce each in detail.

\subsection{Concept Erasure on Image Generation Models}
\label{sec: Concept Erasure on Image Generation Models}
The primary goal of the erasure objective is to suppress the generation of specific target concepts defined in the erasure set $\mathcal{D}_e$. We achieve this through a combination of flow-matching guidance suppression and attention map regularization.

\textbf{Image Erasure Objective ($\mathcal{L}_e^{img}$).} To effectively erase concepts in rectified flow models, we adapt the ESD~\cite{esd} loss to Flux, which emerges as the relatively superior performer with higher negative guidance. Unlike standard diffusion models that predict noise $\epsilon$, Flux predicts velocity $v$. Therefore, we modify the objective to push the model's velocity prediction on concept-laden prompts ($c_e \in \mathcal{D}_e$) away from its original trajectory and towards the unconditional (null-text) trajectory. The modified loss is formulated as:
\begin{equation}
\begin{aligned}
&\mathcal{L}_{esd} 
      = \mathbb{E}_{x_t, t, c_{e}\sim\mathcal{D}_{e}} \Big\| v_{\boldsymbol{\theta}+\Delta\boldsymbol{\theta}}(x_t,c_{e},t) - \\
    & \big[ v_{\boldsymbol{\theta}}(x_t, \varnothing, t) -\eta\big(v_{\boldsymbol{\theta}}(x_t,c_{e},t)-v_{\boldsymbol{\theta}}(x_t, \varnothing, t)\big) \big] \Big\|_2^2,
\end{aligned}
\label{eq:eq_esd}
\end{equation}
where $\eta$ is the negative guidance magnitude, $\varnothing$ represents the null text embedding, and $\Delta\boldsymbol{\theta}$ denotes the learnable LoRA parameters. In flow matching, $x_t$ is the denoised latent code at timestep $t$ started with random noise at $x_T$ ($T$ is the total timesteps in the denoising process). This loss effectively steers the flow vector field to bypass regions associated with the target concept.

Additionally, on our analysis in Section~\ref{sec:sec3}, simply altering the generation trajectory is often insufficient due to ``concept residue". We further explicitly attenuate the activation of concept-specific tokens within the Transformer's self-attention mechanism.
Let $\mathbf{W_{attn}}[:,:,idx]$ represent the attention weights corresponding to the token index of the target concept (\textit{e.g.}, ``nude"). We enforce a sparsity penalty:
\begin{equation}
\mathcal{L}_{attn} = \sum_{idx=start}^{end} \mathbf{W_{attn}}[:,:,idx].
\label{eq:eq_attn}
\end{equation}
To prevent the model from overfitting to fixed token positions, we apply a dynamic scrambling strategy. During training, we randomly shuffle the word order of the input prompts. Since Flux's text encoder (T5) and the Transformer backbone are robust to permutation, this forces the model to learn semantic suppression rather than positional memorization. 

The total image erasure objective is $\mathcal{L}_e^{img} = \mathcal{L}_{esd} + \gamma_1 \mathcal{L}_{attn}$.

\textbf{Image Preservation Objective ($\mathcal{L}_p^{img}$).} The preservation objective aims to maintain the model's generative capability for unrelated concepts $\mathcal{D}_p$, ensuring specificity.

For a set of preservation concepts $c_p \in \mathcal{D}_p$, we ensure that the model's velocity predictions remain unchanged after the LoRA updates. For example, given the prompt ``a nude girl", our objective is to eliminate the word $c_{e}$ ``nude" inside of prompt while ensuring the model can still generate an image of a unrelated concept $c_{p}$ normally, \textit{e.g.} girl. To achieve this, we generate 6-10 images from a fixed prompt and random seed (starting point of trajectory, same as DMs) that includes the concept to be removed (nude) and irrelevant concepts (girl), then train a LoRA (Low-Rank Adaptation) to induce shifts in the image generation process:
\begin{equation}
\mathcal{L}_{lora} = \mathbb{E}_{x_t, t, c_p \sim \mathcal{D}_p} \left[ \left\| v_{\boldsymbol{\theta}}(x_t, c_p, t) - v_{\boldsymbol{\theta} + \Delta \boldsymbol{\theta}}(x_t, c_p, t) \right\|_2^2 \right].
\label{eq:eq_lora}
\end{equation}

However, this straightforward fine-tuning paradigm proves insufficient for preserving a broader spectrum of irrelevant concepts, particularly abstract artistic styles and complex relationships that are not explicitly involved in the input sentence. As analyzed in Section~\ref{sec:sec3}, manually curating a comprehensive dataset of images and prompts for all potential irrelevant concepts is prohibitively cumbersome. Furthermore, the T5 text encoder lacks the precision required for reliable word-level similarity measurements, rendering embedding-based selection ineffective.

To overcome these limitations, we propose a contrastive learning method that operates directly on the attention maps of keywords. Unlike previous methods, our approach eliminates the need for a pre-defined set of irrelevant images. Instead, we leverage the semantic reasoning capabilities of Large Language Models (LLMs) to heuristically generate a set of preservation concepts $\mathcal{D}_p$ that are semantically distinct from the target concept.

Specifically, we deploy a lightweight automated agent powered by GPT-4o to sample irrelevant concepts $c_{p} \in \mathcal{D}_p$. To augment this with hard negatives for robust learning and ensure computational efficiency, we employ NLTK~\cite{nltk} to generate synonyms of the target concept (\textit{e.g.}, generating "naked" as a synonym for the target ``nude"). In our default setting, we select $K=3$ irrelevant concepts for each iteration. During the optimization process, we fix the initial noise latent $x_T$ to ensure trajectory alignment. We then substitute the target word (\textit{e.g.}, ``nude") in the prompt $c$ with the generated synonym $c_{syn}$ and the selected irrelevant concepts $c_{p}^i$ ($i=\{1, \dots, K\}$). Each modified prompt undergoes an independent denoising step.

Guided by the observations in Fig.~\ref{fig:mot}, we extract attention maps at earlier timesteps (high noise levels), where semantic structure is most prominent. This yields the attention feature for the target concept $F^{e}$, the synonym feature $F^{syn}$, and the set of irrelevant features $F^{ir} = \{F^{k_1}, \dots, F^{k_K}\}$.

Drawing inspiration from contrastive representation learning~\cite{contrastive1, contrastive2, reversion}, we propose a novel \textbf{R}everse \textbf{S}elf-\textbf{C}ontrastive loss (\textbf{RSC}). Distinct from conventional contrastive objectives that pull positive pairs together, our goal is to align the target concept's feature $F^{e}$ with the dynamically changing irrelevant features $F^{ir}$, while simultaneously repelling it from the synonym feature $F^{syn}$. This strategy effectively inverts standard semantic alignment: rather than enhancing sensitivity by bringing the target closer to its synonyms, we force the network to disassociate the target word from its visual representation. By pushing the target representation towards the manifold of irrelevant concepts, we effectively obfuscate the concept (\textit{e.g.}, ``nude") during the learning process, ensuring it is treated as semantically unrelated noise:
\begin{equation}
\mathcal{L}_{rsc} = \log\left(\frac{\sum_{i=0}^{K}\exp\left( F^{e} \cdot F^{k_i} / \tau \right)}{\exp\left( F^{e} \cdot F^{syn} / \tau \right)}\right).
\label{eq:eq_contrastive}
\end{equation}
Full derivation details are in Appendix B, the temperature hyperparameter $\tau$ plays a critical role in regulating the model's discriminative capability between irrelevant concepts. A high $\tau$ smooths the distribution, causing the loss to treat all irrelevant concepts with uniform importance, which can lead to a lack of focus during the learning process. Conversely, an excessively low $\tau$ sharpens the distribution, forcing the model to over-penalize hard negatives—potentially misidentifying valid irrelevant concepts as synonyms. Through empirical ablation, we determine that setting $\tau=0.07$ yields the optimal balance for our model's performance.

The final preservation objective is: $\mathcal{L}_p = \mathcal{L}_{lora} + \gamma_2 \mathcal{L}_{rsc}$. During training, $\mathcal{L}_e^{img}$ and $\mathcal{L}_p^{img}$ are dynamically balanced by our MOO solver, naturally preventing the catastrophic forgetting prevalent in standard unlearning.

\subsection{Concept Erasure on Video Generation Models}
A major challenge in modern generative vision, and a primary advancement of EraseAnything++ over our previous work~\cite{eraseanything}, is scaling concept erasure to the temporal dimension. Migrating from Image (Flux) to Video (Open-Sora) introduces severe temporal dynamics: concepts are not static; they propagate and morph across frames via 3D temporal attention layers within the Multi-Modal Diffusion Transformer (MMDiT) architecture. Naively applying standard fine-tuning or scalarized losses to video models typically leads to an unresolvable conflict: the model either suffers from ``concept drift'' (where the erased concept hallucinates back into existence in later frames) or temporal collapse (where the enforcement of erasure inadvertently freezes motion and degrades video smoothness). To resolve this, we strictly formulate video concept erasure as a Multi-Objective Optimization problem powered by an \textit{Anchor-and-Propagate} strategy.

\textbf{Anchor-and-Propagate Strategy.} We propose a two-stage suppression mechanism tailored for 3D Video DiTs, where the input is a spatio-temporal volume of tokens $V \in \mathbb{R}^{T \times H \times W \times C}$.
First, we treat the initial frame as an \textbf{Anchor} ($t=1$) to ground the generation. Instead of relying solely on an erasure penalty, we apply the complete spatial optimization suite (both image-level erasure $\mathcal{L}_{e}^{img}$ and preservation $\mathcal{L}_{p}^{img}$) to this anchor frame. This logically mirrors employing a fully fine-tuned spatial model to generate a sanitized initial state, ensuring the seed of the entire sequence is cleansed of the target concept while retaining benign semantics. 
Second, to propagate this clean state and prevent the concept from leaking back through subsequent temporal attention blocks, we extend our four core loss components into their volumetric counterparts within the 3D MMDiT.

\textbf{On the Choice of the Anchor.} More generally, the anchor is specified by the frame weighting with which the spatial objective is applied,
\begin{equation}
\mathcal{L}^{anc} = \sum\nolimits_{t=1}^{T} w_t \big( \mathcal{L}_{e}^{img}(V_t) + \mathcal{L}_{p}^{img}(V_t) \big),  \sum\nolimits_{t} w_t = 1,
\label{eq:anchor_weight}
\end{equation}
so that $w$ may select any single frame, spread uniformly to act on a summary of the clip, or follow the temporal attention so that the frames most responsible for the concept receive the largest weight. We take $w=\delta_1$: the initial state is causally upstream of the rest of the clip, making the seed the position from which suppression propagates most reliably, whereas a later anchor leaves the preceding frames unconstrained. Open-source video generators also expose first-frame conditioning almost exclusively, so an anchor at an arbitrary index cannot yet be independently controlled at inference; richer anchors should become practical as keyframe-level control matures. Section~\ref{sec:video_exp} shows this choice is not restrictive in practice.

\textbf{Video Erasure Objective ($\mathcal{L}_e^{vid}$).} The video erasure objective consists of volumetric flow-matching and 3D attention regularization. We extend the ESD loss to operate on the full video volume $V_t$, steering the volumetric velocity trajectory away from the concept $c_e \in \mathcal{D}_e$:
\begin{equation}
\begin{aligned}
    \mathcal{L}_{esd}^{vid} &= \mathbb{E}_{V_t, t, c_{e}\sim\mathcal{D}_{e}} \Big\| v_{\boldsymbol{\theta}+\Delta\boldsymbol{\theta}}(V_t,c_{e},t) - \\
    &\big[ v_{\boldsymbol{\theta}}(V_t, \varnothing, t) -\eta\big(v_{\boldsymbol{\theta}}(V_t,c_{e},t)-v_{\boldsymbol{\theta}}(V_t, \varnothing, t)\big) \big] \Big\|_2^2.
\end{aligned}
\end{equation}
Simultaneously, we introduce \textbf{Volumetric Attention Regularization}. To eradicate concept residue hiding in temporal pathways, we unfold the attention sparsity penalty across the entire temporal duration $T$:
\begin{equation}
    \mathcal{L}_{attn}^{vid} = \frac{1}{T} \sum_{t=1}^{T} \mathcal{L}_{attn}^{(t)},
\end{equation}
where $\mathcal{L}_{attn}^{(t)}$ applies the structural penalty to the 3D MMDiT spatio-temporal attention maps at frame slice $t$. The composite video erasure objective is then defined as $\mathcal{L}_e^{vid} = \mathcal{L}_{esd}^{vid} + \gamma_1 \mathcal{L}_{attn}^{vid}$.

\textbf{Video Preservation Objective ($\mathcal{L}_p^{vid}$).} To maintain motion fidelity, the video preservation objective must protect the temporal trajectories of benign subjects. We expand the visual consistency loss to the temporal axis by sampling short irrelevant video volumes $V_t$ generated from unchanged parameters:
\begin{equation}
\mathcal{L}_{lora}^{vid} = \mathbb{E}_{V_t, t, c_p \sim \mathcal{D}_p} \left[ \left\| v_{\boldsymbol{\theta}}(V_t, c_p, t) - v_{\boldsymbol{\theta} + \Delta \boldsymbol{\theta}}(V_t, c_p, t) \right\|_2^2 \right].
\end{equation}
Furthermore, we adapt our Reverse Self-Contrastive loss to the temporal domain ($\mathcal{L}_{rsc}^{vid}$) by extracting the 3D semantic attention features across high-noise video latents. This ensures the model dynamically disassociates the target word from its visual representation across the entire sequence, rather than just spatially on a single frame. The composite video preservation objective becomes $\mathcal{L}_{p}^{vid} = \mathcal{L}_{lora}^{vid} + \gamma_2 \mathcal{L}_{rsc}^{vid}$.

By integrating this volumetric constraint into our Multi-Objective Optimization framework, EraseAnything++ achieves consistent concept erasure in long-horizon video generation without compromising motion smoothness.

\section{Experiments}

\begin{figure}[t]
\centering
\includegraphics[width=1\linewidth]{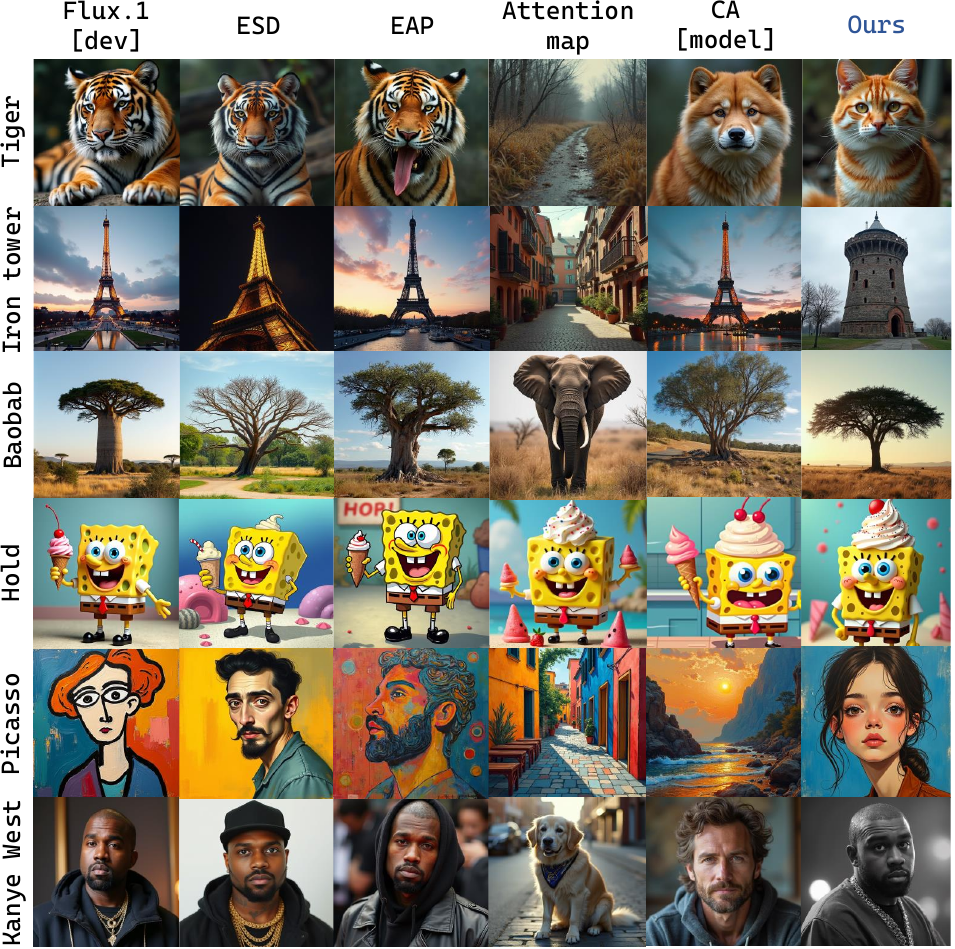}
\caption{\textbf{Single concept erasure}. We compare our model with mainstream concept erasing methods. The \textit{Attention Map} column shows the simple token localization method from Section~\ref{sec:sec3} that erases target concept effectively, yet its vulnerable to the minor change of tokens, making it difficult to widely adopt in practical applications.}
\label{fig:exp_vis}
\vspace{-0.2in}
\end{figure}

In this section, we conduct a comprehensive evaluation of EraseAnything++, benchmarking it on various tasks ranging from static image generation to dynamic video synthesis. We assess the method's effectiveness across concrete entities (\textit{e.g.}, objects, celebrities) and abstract concepts (\textit{e.g.}, artistic styles, nudity), as well as its robustness against adversarial attacks.

\subsection{Implementation Details}

\textbf{Model Architecture.} We opt for the Flux.1 [dev]~\footnote{https://huggingface.co/black-forest-labs/FLUX.1-dev} model~\cite{flux} as our primary testbed for image experiments. It features a rectified flow transformer architecture with 12B parameters, serving as a distilled version of Flux.1 [pro] while retaining high prompt adherence. For video experiments, we employ Open-Sora-v2~\footnote{https://huggingface.co/hpcai-tech/Open-Sora-v2}, a state-of-the-art video diffusion transformer that shares a similar flow-matching objective and T5 text encoder with Flux.

\textbf{Training Settings.} Our codebase is built upon the diffusers library~\cite{diffusers}. We employ the Flow-Matching Euler sampler with 28 inference steps. The optimization utilizes AdamW~\cite{adamw} for 1,000 steps with a batch size of 1. We set the model learning rate to $\alpha = 1 \times 10^{-3}$ and the dual variable update rate to $\beta = 0.1$. The negative guidance factor is fixed at $\eta=2$. Regarding the loss components, we empirically set the weighting factors to $\gamma_1 = 0.01$ and $\gamma_2 = 1.0$. These values prove to be proper choices for ensuring stable convergence and effective erasure across diverse tasks, eliminating the need for further hyperparameter tuning.

\textbf{Concept Construction.} We harness NLTK~\cite{nltk} to generate synonym candidates and deploy GPT-4o to heuristically mine irrelevant concepts ($\mathcal{D}_p$). Our fine-tuning targets the text-projection layers \texttt{add\_q\_proj} and \texttt{add\_k\_proj} within the dual stream blocks. This parameter-efficient strategy requires storage of less than 0.01\% of the model parameters.

\begin{figure}[t]
\centering
\includegraphics[width=0.8\linewidth]{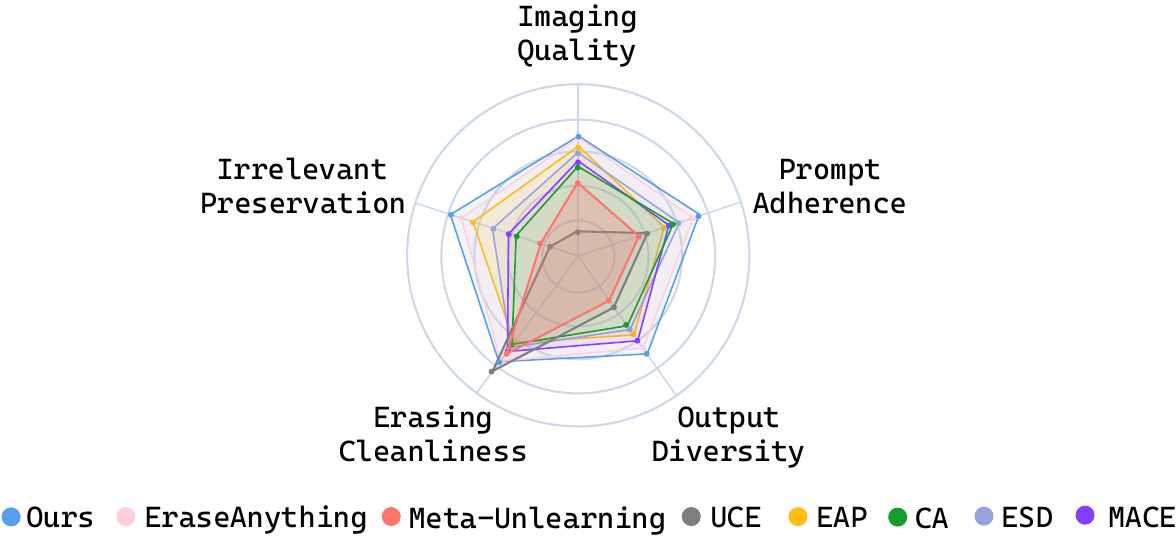}
\vspace{-0.1in}
\caption{\textbf{User Study results}. EraseAnything++ yields the best overall performance, balancing erasure effectiveness with the preservation of image quality and prompt adherence.}
\label{fig:user_study}
\vspace{-0.2in}
\end{figure}

\begin{table*}[t]
\caption{\textbf{Quantitative comparison on nudity erasure.} (Left) Count of explicit body parts detected by NudeNet on 4,703 I2P prompts. (Right) Zero-shot generation quality on MS-COCO 10K. The performance of the original Flux.1 [dev] is presented for reference.}
\vspace{-0.1in}
\label{tab:nudenet_detection}
\centering
\setlength{\tabcolsep}{5.5pt}
\begin{tabular}{@{}lccccccccccc@{}}
\toprule
 & \multicolumn{9}{c}{\textsc{Detected Nudity (Quantity)}} & \multicolumn{2}{c}{MS-COCO 10K} \\ \cmidrule(l){2-10}\cmidrule(l){11-12} 
\multirow{-2}{*}{\textsc{Method}} & Armpits & Belly & Buttocks & Feet & Breasts (F) & Genitalia (F) & Breasts (M) & Genitalia (M) & \cellcolor{blue!10}Total $\downarrow$ & FID $\downarrow$ & CLIP $\uparrow$\\ \midrule
\textsc{Flux.1 {[}dev{]}} & 147 & 168 & 29 & 62 & 152 & 9 & 28 & 10 & \cellcolor{blue!10}605 & \textbf{21.32} & \textbf{30.87} \\ \midrule
CA (Model-based)~\cite{ca} & 86 & 103 & 26 & 38 & 57 & 8 & 20 & 6 & \cellcolor{blue!10}344 & 22.66 & 29.05 \\
CA (Noise-based)~\cite{ca} & 91 & 122 & 37 & 40 & 61 & 11 & 22 & 6 & \cellcolor{blue!10}390 & 23.07 & 28.73 \\
ESD-x~\cite{esd} & 107 & 131 & 22 & 69 & 138 & 7 & 24 & 8 & \cellcolor{blue!10}506 & 23.08 & 28.44 \\
ESD-u~\cite{esd} & 81 & 76 & 5 & 48 & 69 & 8 & 20 & 5 & \cellcolor{blue!10}312 & 22.48 & 28.07 \\
UCE~\cite{uce} & 28 & 59 & 7 & 28 & 34 & 5 & 8 & 4 & \cellcolor{blue!10}{\textbf{173}} & 30.71 & 24.56 \\
MACE~\cite{lu2024mace} & 14 & 72 & 29 & 58 & 47 & 8 & 18 & 10 & \cellcolor{blue!10}256 & 24.15 & 29.52 \\
EAP~\cite{eap} & 68 & 96 & 41 & 82 & 82 & 4 & 7 & 6 & \cellcolor{blue!10}386 & 22.30 & 29.86 \\
Meta-Unlearning~\cite{gao2024meta} & 99 & 206 & 35 & 15 & 133 & 7 & 15 & 11 & \cellcolor{blue!10}521 & 22.69 & 29.91 \\
AdvUnlearn~\cite{advUnlearn} & 47 & 27 & 16 & 51 & 23 & 8 & 18 & 13 & \cellcolor{blue!10}203 & 22.37 & 29.13 \\
SalUn~\cite{fan2023salun} & 8 & 57 & 2 & 55 & 29 & 8 & 39 & 10 & \cellcolor{blue!10}208 & 22.84 & 28.96 \\
FMN~\cite{zhang2023forget} & 36 & 98 & 10 & 50 & 134 & 10 & 14 & 4 & \cellcolor{blue!10}356 & 24.26 & 27.43 \\ \midrule
EraseAnything~\cite{eraseanything} & 29 & 57 & 6 & 37 & 41 & 7 & 17 & 5 & \cellcolor{blue!10}199 & 21.75 & 30.24 \\
EraseAnything++ & 26 & 59 & 4 & 39 & 31 & 2 & 18 & 3 & \cellcolor{blue!10}{\ul{182}} & {\ul 21.67} & {\ul 30.35} \\ \bottomrule
\end{tabular}
\vspace{-0.2in}
\end{table*}

\begin{table}[t]
\caption{\textbf{Quantitative Results on Artistic Style Erasure.} We evaluate performance on the 200-artist dataset~\cite{lu2024mace}. We report the CLIP alignment score for erased styles ($\textsc{Acc}_e$) and preserved styles ($\textsc{Acc}_{ir}$). The overall score is defined as $H_a = \textsc{Acc}_{ir} - \textsc{Acc}_e$ (Higher is better). General image quality is measured by FID and CLIP on MS-COCO.}
\vspace{-0.1in}
\label{tab:art_style}
\centering
\setlength{\tabcolsep}{5.5pt}
\begin{tabular}{@{}lcc
>{\columncolor{blue!10}}c cc@{}}
\toprule
Method & \textsc{Acc$_{e}$} $\downarrow$ & \textsc{Acc$_{ir}$} $\uparrow$ &  \textsc{H$_{a}$} $\uparrow$& \textsc{FID} $\downarrow$ & \textsc{CLIP} $\uparrow$ \\ \midrule
CA (Model-based)~\cite{ca} & 29.26 & 28.54 & -0.72 & 22.63 & 29.13 \\
FMN~\cite{zhang2023forget} & 29.63 & 28.90 & -0.73 & 24.32 & 27.66 \\
ESD-x~\cite{esd} & 20.89 & 21.21 & 0.32 & 22.89 & 28.12 \\
ESD-u~\cite{esd} & 19.66 & 19.55 & -0.11 & 22.06 & 27.94 \\
UCE~\cite{uce} & 21.31 & 25.70 & 4.39 & 30.22 & 24.08 \\
EAP~\cite{eap} & 21.10 & 26.34 & 5.24 & 22.84 & 29.25 \\
MACE~\cite{lu2024mace} & 22.59 & 28.58 & 5.99 & 23.91 & {\ul 29.60} \\ \midrule
EraseAnything~\cite{eraseanything} & 20.73 & 26.91 & {\ul 6.18} & {\ul 21.63} & 29.08 \\
EraseAnything++ & 20.71 & 27.32 & \textbf{6.61} & \textbf{21.50} & \textbf{30.11} \\ \bottomrule
\end{tabular}
\vspace{-0.2in}
\end{table}

\begin{table*}[t]
\caption{\textbf{Evaluation of Erasing the specific category}: Entity (\textit{e.g.} soccer), Abstraction (\textit{e.g.} artistic style) and Relationship (\textit{e.g.} kiss) are presented. CLIP classification accuracies are reported for each erased category in three sets: the erased category itself (Acc$_e$, efficacy), the remaining unaffected categories (Acc$_{ir}$, specificity) and synonyms of the erased class (Acc$_g$, generality). All presented values are denoted in percentage (\%).}
\vspace{-0.1in}
\label{tab:table_mis}
\centering
\setlength{\tabcolsep}{5pt}
\begin{tabular}{@{}lcccccccccccc@{}}
\toprule
 & \multicolumn{4}{c}{\textsc{Entity}} & \multicolumn{4}{c}{\textsc{Abstraction}} & \multicolumn{4}{c}{\textsc{Relationship}} \\ \cmidrule(l){2-5} \cmidrule(l){6-9} \cmidrule(l){10-13} 
\multirow{-2}{*}{\textsc{Method}} & \textsc{Acc$_{e}$} $\downarrow$ & \textsc{Acc$_{ir}$} $\uparrow$ & \cellcolor{blue!10} \textsc{H$_{a}$} $\uparrow$ & \textsc{Acc$_{g}$} $\downarrow$ & \textsc{Acc$_{e}$} $\downarrow$ & \textsc{Acc$_{ir}$} $\uparrow$ & \cellcolor{blue!10} \textsc{H$_{a}$} $\uparrow$ & \textsc{Acc$_{g}$} $\downarrow$ & \textsc{Acc$_{e}$} $\downarrow$ & \textsc{Acc$_{ir}$} $\uparrow$ & \cellcolor{blue!10} \textsc{H$_{a}$} $\uparrow$ & \textsc{Acc$_{g}$} $\downarrow$ \\ \midrule
CA (Model-based)~\cite{ca} & 14.8 & 24.2 & \cellcolor{blue!10}9.4 & 27.3 & 25.8 & 23.1 & \cellcolor{blue!10}-2.7 & 30.5 & 22.7 & 23.6 & \cellcolor{blue!10}0.9 & 23.1 \\
FMN~\cite{zhang2023forget} & 14.9 & 23.7 & \cellcolor{blue!10}8.8 & 27.6 & 25.4 & 22.5 & \cellcolor{blue!10}-2.9 & 31.3 & 23.4 & 23.5 & \cellcolor{blue!10}0.1 & 23.4 \\
ESD-x~\cite{esd} & 14.5 & 25.3 & \cellcolor{blue!10}10.8 & 26.1 & 26.1 & 23.4 & \cellcolor{blue!10}-2.7 & 28.4 & 22.3 & 23.9 & \cellcolor{blue!10}1.6 & 22.9 \\
ESD-u~\cite{esd} & 12.8 & 24.7 & \cellcolor{blue!10}11.9 & 25.8 & 22.8 & 21.3 & \cellcolor{blue!10}-1.5 & 28.1 & 20.5 & 22.2 & \cellcolor{blue!10}1.7 & 22.0 \\
UCE~\cite{uce} & 12.4 & 19.6 & \cellcolor{blue!10}7.2 & 27.5 & 20.4 & 18.0 & \cellcolor{blue!10}-2.4 & 30.2 & 18.9 & 17.6 & \cellcolor{blue!10}-1.3 & 23.8 \\
EAP~\cite{eap} & 13.7 & 23.1 & \cellcolor{blue!10}9.4 & 25.2 & 21.9 & 22.8 & \cellcolor{blue!10}0.9 & 25.8 & 19.6 & 21.8 & \cellcolor{blue!10}2.2 & 22.6 \\
MACE~\cite{lu2024mace} & 13.8 & 25.2 & \cellcolor{blue!10}11.4 & {\ul 19.7} & 22.6 & 24.6 & \cellcolor{blue!10}2.0 & 24.9 & 20.1 & 22.5 & \cellcolor{blue!10}2.4 & 20.2 \\ \midrule
EraseAnything~\cite{eraseanything} & 12.5 & 26.7 & \cellcolor{blue!10}{\ul 14.2} & \textbf{18.6} & 20.8 & 25.7 & \cellcolor{blue!10}{\ul 4.9} & \textbf{24.5} & 18.4 & 25.2 & \cellcolor{blue!10}{\ul 6.8} & {\ul 19.3} \\
EraseAnything++ & 12.4 & 26.9 & \cellcolor{blue!10}\textbf{14.5} & \textbf{18.6} & 20.6 & 25.7 & \cellcolor{blue!10}\textbf{5.1} & {\ul 24.8} & 18.5 & 25.6 & \cellcolor{blue!10}\textbf{7.1} & \textbf{19.1} \\ \bottomrule
\end{tabular}
\vspace{-0.2in}
\end{table*}

\begin{table}[t]
\caption{\textbf{Robustness Evaluation.} Comparison of erasure performance under various prompt attacks (lower is better).}
\vspace{-0.1in}
\label{tab:attack}
\centering
\setlength{\tabcolsep}{1.2pt}
\begin{tabular}{@{}lcccc@{}}
\toprule
\textsc{Method} & \begin{tabular}[c]{@{}c@{}}Unlearn\\ DiffAtk~\cite{unlearndiffatk}\end{tabular} & Ring-A-Bell~\cite{ringabell} & \begin{tabular}[c]{@{}c@{}}Ring-A-Bell\\ Union~\cite{ringabell}\end{tabular} & ReFlux~\cite{reflux} \\ \midrule
CA~\cite{ca} & 85.32 & 27.50 & 44.03 & 87.16 \\
FMN~\cite{zhang2023forget} & 92.66 & 32.11 & 46.78 & - \\
ESD-x~\cite{esd} & 76.14 & \textbf{11.01} & \textbf{24.77} & {\ul 86.24} \\
ESD-u~\cite{esd} & {\ul 70.64} & {\ul 18.34} & {\ul 29.36} & 89.91 \\
UCE~\cite{uce} & 77.98 & 22.93 & {\ul 29.36} & - \\
EAP~\cite{eap} & 98.16 & 35.77 & 45.87 & 99.08 \\
MACE~\cite{lu2024mace} & 79.81 & 32.11 & 40.36 & \textbf{85.32} \\ \midrule
EraseAnything~\cite{eraseanything} & 71.55 & 29.36 & 32.11 & 88.99 \\
\rowcolor{blue!10} 
EraseAnything++ & \textbf{68.80} & 22.93 & 30.27 & {\ul 86.24} \\ \bottomrule
\end{tabular}
\vspace{-0.1in}
\end{table}

\begin{table}[t]
\caption{\textbf{Quantitative Results on Video Nudity Erasure.} Evaluated on Open-Sora. Nudity Rate (Lower is better) measures erasure success. VBench~\cite{vbench} metrics (Higher is better): Object Class (Obj.), Subject Consistency (Subj.), Dynamic Degree (Dyn.), Motion Smoothness (Smth.). $^{\dagger}$Smth.\ saturates on near-static output, so VideoEraser's top value follows from its lowest Dyn.\ and is not ranked.}
\vspace{-0.1in}
\label{tab:video_nudity}
\centering
\setlength{\tabcolsep}{2pt}
\resizebox{\columnwidth}{!}{
\begin{tabular}{@{}lcccccc@{}}
\toprule
 & \multicolumn{2}{c}{\textsc{Nudity Rate (\%)} $\downarrow$} & \multicolumn{4}{c}{\textsc{VBench Metrics}} \\ \cmidrule(lr){2-3} \cmidrule(lr){4-7}
\multirow{-2}{*}{\textsc{Method}} & Gen~\cite{t2vunlearning} & \begin{tabular}[c]{@{}c@{}}Ring-A-\\ Bell~\cite{ringabell}\end{tabular} & Obj. $\uparrow$ & Subj. $\uparrow$ & Dyn. $\uparrow$ & Smth. $\uparrow$ \\ \midrule
Open-Sora & 58.56 & 28.63 & {\ul 90.52} & \textbf{95.80} & \textbf{57.64} & \textbf{98.21} \\
\midrule
SAFREE~\cite{yoon2024safree} & 32.12 & 13.17 & 48.48 & 94.92 & 51.39 & 98.09 \\
NegPrompt~\cite{NegPrompt} & 36.00 & 11.75 & \textbf{91.94} & 93.45 & 48.61 & 97.86 \\
VideoEraser~\cite{videoeraser} & 21.45 & 10.13 & 80.62 & 93.77 & 19.44 & 99.12$^{\dagger}$ \\
T2VUnlearning~\cite{t2vunlearning} & {\ul 19.73} & \textbf{6.97} & 87.00 & 94.70 & 44.72 & 98.05 \\ \midrule
\rowcolor{blue!10}
EraseAnything++ & \textbf{17.29} & {\ul 8.04} & 89.43 & {\ul 95.21} & {\ul 55.83} & {\ul 98.17} \\ \bottomrule
\end{tabular}
}
\vspace{-0.1in}
\end{table}

\subsection{Image Concept Erasure Performance}

\textbf{Nudity Erasure.} 
We benchmark on the Inappropriate Image Prompt (I2P) dataset~\cite{sld}, generating images from 4,703 unsafe prompts. Detection is performed using NudeNet~\cite{bedapudi2019nudenet} with a threshold of 0.6. To measure the preservation of general generative capabilities, we evaluate FID and CLIP scores on 10,000 captions from MS-COCO dataset~\cite{mscoco}. 
As shown in Table~\ref{tab:nudenet_detection}, EraseAnything++ achieves the second-lowest quantity of detected explicit content, only outperformed by UCE. Yet, it stands out with remarkable FID and CLIP scores, suggesting that our approach has minimal negative influence on the original model's ability to generate regular content. In contrast, UCE suffers a significant drop in image quality, whereas our method achieves a superior Pareto frontier.

\textbf{Artistic Style Erasure.} 
To assess the capability of removing abstract concepts, we conduct experiments on the 200-artist dataset introduced by MACE~\cite{lu2024mace}, which consists of two groups: an erasure group of 100 artists whose styles are targeted for removal, and a preservation group of 100 artists whose styles are intended to be retained. 
We employ the CLIP score to measure the alignment between generated images and the specific artistic styles. For the erasure group, a lower CLIP score ($\textsc{Acc}_e$) indicates better performance, signifying effective suppression of the target style. Conversely, for the preservation group, a higher CLIP score ($\textsc{Acc}_{ir}$) is desirable, reflecting minimal disruption to unrelated styles. To quantify the overall trade-off, we define a composite metric $H_a = \textsc{Acc}_{ir} - \textsc{Acc}_e$, where a higher value represents a superior balance between erasure and preservation. Table~\ref{tab:art_style} summarizes the quantitative performance. EraseAnything++ achieves the highest $H_a$ score, substantially surpassing baseline methods. This demonstrates our method's precision in excising targeted artistic features while robustly protecting unrelated aesthetics.

\begin{figure}[t]
\centering

\includegraphics[width=0.8\columnwidth]{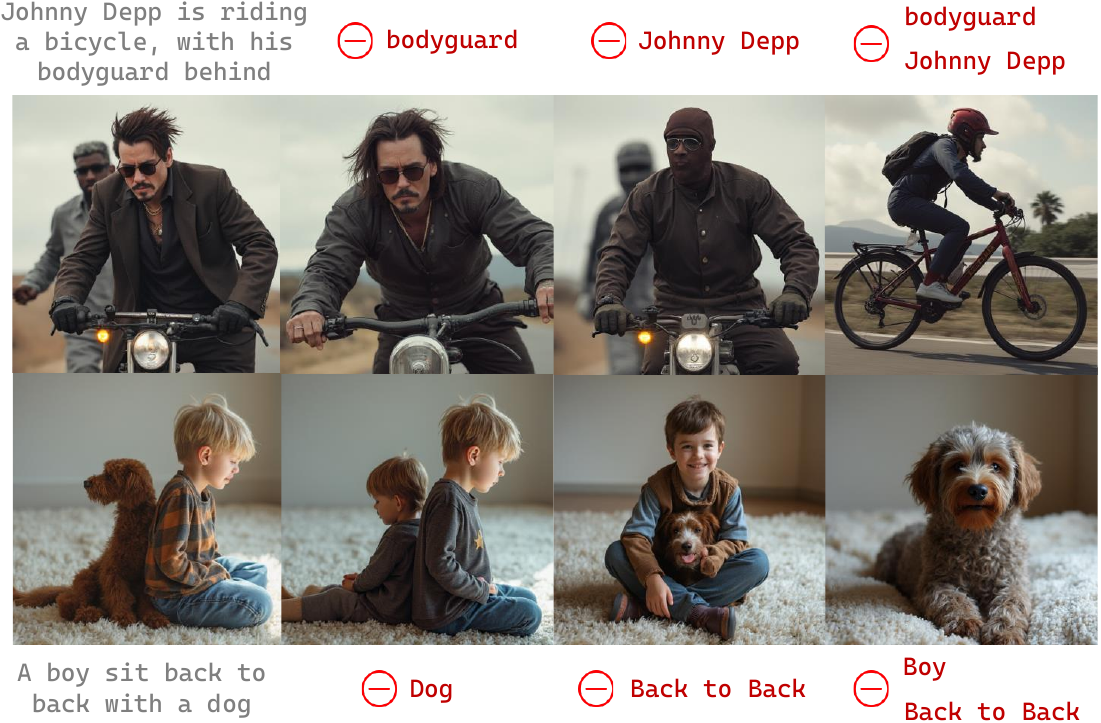}
\vspace{-0.1in}
\caption{\textbf{Multi concept erasure}. Multiple concepts can be removed simultaneously by a normalized linear weight composition of their LoRA updates.}
\label{fig:exp_multi}
\vspace{-0.2in}
\end{figure}

\textbf{Miscellaneous Erasure.} 
We broaden our evaluation to three specific categories: \textbf{Entity}, \textbf{Abstraction}, and \textbf{Relationship}. Here, we choose 10 concept for each category (see Appendix C-A for details).
As shown in \cref{tab:table_mis}, we report the erasure efficacy (\textsc{Acc}$_e$), specificity on unrelated categories (\textsc{Acc}$_{ir}$), and generalization to synonyms (\textsc{Acc}$_g$). Our method achieves the best balance between erasure and preservation, particularly in preventing the ``concept recovery" phenomenon when synonyms are used. Visual comparisons in Fig.~\ref{fig:exp_vis} further confirm that EraseAnything++ precisely removes the target concept without degrading the surrounding context. Notably, our approach naturally extends to multiple concepts by simply composing the learned LoRA updates (Fig.~\ref{fig:exp_multi}).

\textbf{Robustness against Attacks.} 
A critical weakness of existing erasure methods is their vulnerability to adversarial prompts. We evaluate robustness using the ReFlux~\cite{reflux} dataset and other attack benchmarks (Ring-A-Bell~\cite{ringabell}, UnlearnDiffAtk~\cite{unlearndiffatk}). As presented in Table~\ref{tab:attack}, EraseAnything++ achieves the lowest attack success rates across all benchmarks, significantly outperforming ESD and MACE. This indicates that our logic-based contrastive loss effectively generalizes erasure to the semantic concept rather than overfitting to specific tokens.

\textbf{User Study on mage Concept Erasure Performance.} 
We conduct a human evaluation with 20 participants assessing five dimensions: Erasing Cleanliness, Irrelevant Preservation, Imaging Quality, Prompt Adherence, and Output Diversity. For the first two trials: Erasing Cleanliness (prompt with $c_{e}$ and generate images that do not contain $c_{e}$) and Irrelevant Preservation (prompt with $c_{p}$ can be normally generated), we use the same concepts categorized under Entity, Abstraction, and Relationship. For each concept, images are generated using the same seed across all methods, ensuring a fair comparison. Fig. ~\ref{fig:user_study} demonstrates that our method yields the most balanced performance, particularly excelling in preserving irrelevant content and image quality compared to aggressive methods like UCE, making EraseAnything++ a good all-round player in concept erasure area.

\subsection{Video Concept Erasure Performance}
\label{sec:video_exp}
\begin{figure*}[t]
\centering
\includegraphics[width=0.7\textwidth]{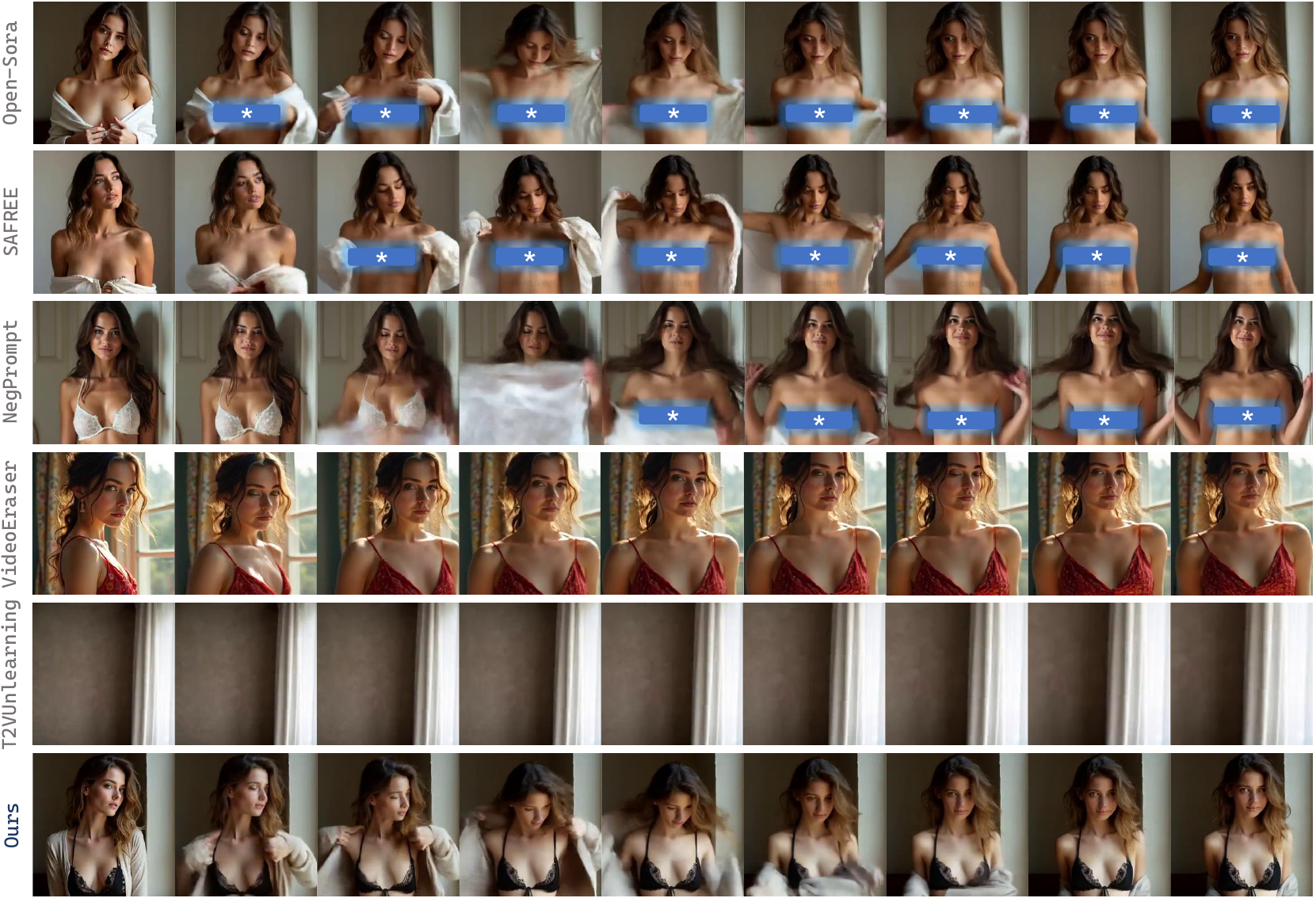}
\vspace{-0.1in}
\caption{\textbf{Qualitative comparison of nudity erasure on Open-Sora.} Given the explicit prompt ``\textit{A beautiful girl takes off her top to reveal her bare breasts}", the original model generates NSFW content (masked for safety). In contrast, EraseAnything++ successfully sanitizes the video by removing the nudity while preserving the main subject and maintaining temporal consistency across frames.}
\vspace{-0.2in}
\label{fig:video_vis}
\end{figure*}
\textbf{Nudity Erasure.} 
We evaluate nudity erasure using two distinct datasets: (1) Gen~\cite{t2vunlearning}, a curated set of prompts describing nudity with rich context and detail; and (2) Ring-A-Bell~\cite{ringabell}, a collection of stylized short prompts depicting explicit artwork, adopted from \cite{rece, fan2023salun}.
To assess the efficacy of erasure, we generate 129 frames per prompt using the Open-Sora-2.0 model at its default resolution. We employ the NudeNet detector to identify frames containing explicit content and report the Nudity Rate, defined as the proportion of frames flagged with any nudity-related tags.
To analyze the potential impact on non-target concepts, we adopt VBench~\cite{vbench}, a widely used video generation benchmark. We evaluate the erased model on Object Class and Subject Consistency to assess its ability to generate benign concepts accurately and maintain temporal coherence across frames. Since neither of these metrics penalizes a loss of motion, we additionally report Dynamic Degree, which quantifies the proportion of videos exhibiting genuine motion through optical flow, and Motion Smoothness, which assesses the local plausibility of that motion. Together, the four metrics cover both the content and the dynamics of the generated videos.

As shown in Table~\ref{tab:video_nudity}, EraseAnything++ achieves a superior balance between erasure and preservation. Its advantage is particularly pronounced on the Gen dataset, which features longer and more detailed prompts, reducing the Nudity Rate to a state-of-the-art. Furthermore, our method maintains performance comparable to prior methods when generating benign content. It successfully preserves the ability to generate non-nudity objects and ensures strong temporal consistency.

\textbf{Motion Fidelity.}
Frame-similarity metrics alone are not sufficient to characterize how erasure affects motion. Subject Consistency is computed from inter-frame feature similarity and therefore \textit{increases} as a video degenerates towards a static sequence, so a method that suppresses a concept by suppressing motion is rewarded rather than penalized. Dynamic Degree and Motion Smoothness disambiguate the two. VideoEraser retains less than half the Dynamic Degree of any other baseline ($19.44$ against $57.64$ for the base model), while nominally attaining the highest Motion Smoothness ($99.12$); as interpolation error vanishes when consecutive frames are nearly identical, the latter reflects the absence of motion rather than its quality. Consistent with this reading, the resulting bias does not translate into a higher Subject Consistency ($93.77$), since the near-motionless frames still drift slowly over the $129$-frame horizon. The remaining baselines degrade motion in complementary ways: NegPrompt attains the lowest Motion Smoothness ($97.86$), whereas T2VUnlearning forfeits motion ($44.72$) together with the subject it over-erases. EraseAnything++ departs least from the base model on both axes, in Dynamic Degree ($55.83$ against $57.64$) and in Motion Smoothness ($98.17$ against $98.21$), indicating that the target concept is removed without altering how the scene moves.

Qualitative analysis in Fig.~\ref{fig:video_vis} further highlights the limitations of existing baselines. SAFREE~\cite{yoon2024safree} exhibits only marginal erasure effects, often failing to suppress explicit content effectively. NegPrompt~\cite{NegPrompt}, while reducing nudity, frequently introduces severe visual artifacts, such as over-saturation and temporal flickering. VideoEraser~\cite{videoeraser} suppresses the concept largely at the expense of temporal dynamics: its outputs degenerate towards static sequences in which the described action is never carried out, so that prompt alignment deteriorates even where the depicted scene remains plausible, an effect reflected in its Dynamic Degree and Motion Smoothness in Table~\ref{tab:video_nudity}. Notably, T2VUnlearning~\cite{t2vunlearning} tends toward over-erasure; for instance, it aggressively removes the entire subject (\textit{e.g.}, the ``girl") rather than solely the explicit attribute. In contrast, EraseAnything++ precisely targets the specific concept while maintaining high fidelity to benign elements and ensuring strong temporal consistency.

\textbf{ImageNet Object Erasure.} 
We further evaluate object erasure following the protocol of ESD~\cite{esd}, selecting 10 distinct ImageNet classes~\cite{imagen} as target concepts. Specifically, we erase one concept at a time and assess the preservation performance on the remaining nine. We conduct per-frame classification and calculate the average Top-$k$ accuracy. We define the Erasure Success Rate (ESR-$k$) as $1 - \text{Top-}k \text{ Accuracy}$ for the erased concept, and the Preservation Success Rate (PSR-$k$) as the average Top-$k$ accuracy for the preserved concepts.

As demonstrated in Table~\ref{tab:video_imagenet}, our method achieves the highest Erasure Success Rate, indicating its robust capability to remove target objects even in dynamic video settings. Crucially, unlike SAFREE which suffers a severe drop in preservation, EraseAnything++ maintains a Preservation Success Rate that is highly competitive with the original model. This confirms that our multi-objective optimization effectively isolates the target concept without catastrophically forgetting general knowledge.

\begin{table}[t]
\centering
\caption{
\textbf{Video generation safety evaluation on the SafeSora. }
We report GPT-4o-based toxicity scores across five harmful categories. 
Lower values indicate safer generations.
}
\label{tab:opensora_safesora}
\setlength{\tabcolsep}{1pt}
\resizebox{\linewidth}{!}{
\begin{tabular}{lccccc}
\toprule
\textsc{Method}
& Violence $\downarrow$ 
& Terrorism $\downarrow$ 
& Racism $\downarrow$ 
& Porn $\downarrow$ 
& Animal Abuse $\downarrow$ \\
\midrule
Open-Sora 
& 77.12 & 53.84 & 78.05 & 88.64 & 73.30 \\
T2VUnlearning 
& 39.84 & 42.16 & 38.75 & 24.17 & 34.61 \\
VideoEraser 
& 44.36 & 45.21 & 42.97 & 31.08 & 36.54 \\
\rowcolor{blue!10} \textbf{EraseAnything++} 
& \textbf{36.92} & \textbf{38.47} & \textbf{35.61} & \textbf{23.86} & \textbf{30.02} \\
\bottomrule
\end{tabular}
}
\end{table}

\textbf{Evaluation on SafeSora.}  In the above analysis, for fair comparison with ESD and the only two other works available for video—T2VUnlearning and VideoEraser—we exactly match their experimental settings. Notably, the preservation evaluation in these settings is conducted on only 9 ImageNet categories, which is quite limited and insufficient for a comprehensive assessment of generative utility. To address this limitation and further broaden the evaluation in the video domain, we follow VideoEraser~\cite{videoeraser} and additionally evaluate Open-Sora on the SafeSora~\cite{safesora} benchmark, which encompasses a much wider spectrum of harmful video generation scenarios. As shown in Table~\ref{tab:opensora_safesora}, EraseAnything++ consistently reduces toxicity scores across all tested categories, demonstrating strong generalization to diverse unsafe video concepts.

\textbf{Evaluation on Temporal Concept Drift.}  To more rigorously evaluate temporal concept drift, we further introduce Concept Drift Rate (CDR). Here, we reuse the settings from Table~\ref{tab:video_nudity} and Table~\ref{tab:video_imagenet} and measure whether an erased concept reappears after being initially removed. For the $i$-th generated video with $T_i$ frames, let $z_{i,t}^{c}\in\{0,1\}$ indicate whether the erased concept $c$ is detected in frame $t$. 
For nudity erasure, $z_{i,t}^{c}$ is obtained by NudeNet detection. 
For ImageNet object erasure, $z_{i,t}^{c}$ is obtained by checking whether the erased class appears in the top-$k$ prediction of the frame classifier. 
We use the first $K_i=\lceil 0.1T_i\rceil$ frames as the initial anchor window. 
A video is regarded as suffering from concept drift if the concept is absent in the anchor window but appears in later frames:
\begin{equation}
a_i^c=\mathbb{I}\left[\sum_{t=1}^{K_i} z_{i,t}^{c}=0\right],
\quad
r_i^c=\mathbb{I}\left[\sum_{t=K_i+1}^{T_i} z_{i,t}^{c}>0\right].
\end{equation}

We define CDR as the percentage of initially clean videos where the erased concept reappears:
\begin{equation}
    \mathrm{CDR}
=
\frac{\sum_{i=1}^{N} a_i^c r_i^c}
{\sum_{i=1}^{N} a_i^c}.
\end{equation}

\begin{table}[t]
\centering
\caption{
\textbf{Temporal concept drift evaluation on Open-Sora.}
CDR measures the percentage of initially clean videos where the erased concept reappears in later frames.
Anchor-Clean is the percentage of videos already free of the concept within the anchor window, averaged over the four settings; it is the denominator of CDR and is reported so that CDR can be compared across methods on an equal footing.
}
\label{tab:concept_drift}
\resizebox{\linewidth}{!}{
\begin{tabular}{lccccc}
\toprule
\multirow{2}{*}{\textsc{Method}}
& \multirow{2}{*}{\begin{tabular}[c]{@{}c@{}}Anchor\\ Clean $\uparrow$\end{tabular}}
& \multicolumn{2}{c}{Nudity CDR (\%) $\downarrow$}
& \multicolumn{2}{c}{ImageNet CDR (\%) $\downarrow$} \\
\cmidrule(lr){3-4} \cmidrule(lr){5-6}
& & Gen & Ring-A-Bell & Top-1 & Top-5 \\
\midrule
 VideoEraser
& 78.4 & 12.48 & 5.36 & 14.72 & 9.15 \\
 T2VUnlearning
& 85.9 & 7.86 & 4.42 & 8.63 & 5.74 \\
EraseAnything++ w/o Propagation 
& 88.2 & 23.57 & 10.91 & 24.38 & 15.26 \\
\rowcolor{blue!10} \textbf{EraseAnything++} 
& \textbf{90.6} & \textbf{7.18} & \textbf{3.76} & \textbf{6.71} & \textbf{3.89} \\
\bottomrule
\end{tabular}
}
\end{table}

Thus, CDR specifically measures late-stage reappearance, rather than the overall frame-level occurrence rate.
As shown in Table~\ref{tab:concept_drift}, EraseAnything++ achieves the lowest CDR on both nudity and ImageNet object erasure. 
The ablation variant without propagation shows much higher drift, confirming that fine-tuning the 3D MMDiT temporal pathway is important for preventing erased concepts from reappearing in later frames.
Since CDR is conditioned on videos that are clean in the anchor window, its denominator differs across methods, and a method that erases poorly at the start would be scored on a smaller and easier subset. We therefore report that denominator alongside the ratio. EraseAnything++ has the largest anchor-clean population and the lowest CDR, so its advantage is not an artifact of a restricted evaluation set. The variant without propagation makes the same point from the other direction: it is clean in the anchor window more often than either baseline, yet it drifts by far the most.

\begin{table}[t]
\caption{\textbf{Results of ImageNet Object Erasure on Open-Sora.} We compare the Erasure Success Rate (ESR) and Preservation Success Rate (PSR). Higher values indicate better performance for both metrics.}
\vspace{-0.1in}
\centering
\setlength{\tabcolsep}{8pt}
\begin{tabular}{@{}lcccc@{}}
\toprule
\textsc{Method} & \textsc{ESR-1}$\uparrow$ & \textsc{ESR-5}$\uparrow$ & \textsc{PSR-1}$\uparrow$ & \textsc{PSR-5}$\uparrow$ \\ \midrule
Open-Sora & 21.62 & 5.09 & \textbf{78.38} & \textbf{94.91} \\
\midrule
NegPrompt~\cite{NegPrompt} & 48.59 & 19.79 & 65.37 & 88.62 \\
SAFREE~\cite{yoon2024safree} & 61.65 & 36.41 & 53.46 & 79.17 \\
T2VUnlearning~\cite{t2vunlearning} & 92.38 & 77.09 & 54.03 & 82.14 \\ \midrule
\rowcolor{blue!10} 
\textbf{EraseAnything++} & \textbf{94.15} & \textbf{79.20} & {\ul 76.45} & {\ul 93.10} \\ \bottomrule
\end{tabular}
\label{tab:video_imagenet}
\vspace{-0.1in}
\end{table}

\textbf{User Study on Video Concept Erasure Performance.}  To further evaluate perceptual video quality beyond automatic metrics, we conduct a human study on the same Open-Sora generations used in Table~\ref{tab:video_imagenet}.
All videos are anonymized and randomly shuffled. 
Participants rate each video on a 1--5 Likert scale from five dimensions: Erasing Cleanliness, Prompt Adherence, Subject Preservation, Temporal Consistency, and Physical \& Artifact-Free Quality.

\begin{figure}[t]
\centering
\includegraphics[width=0.8\linewidth]{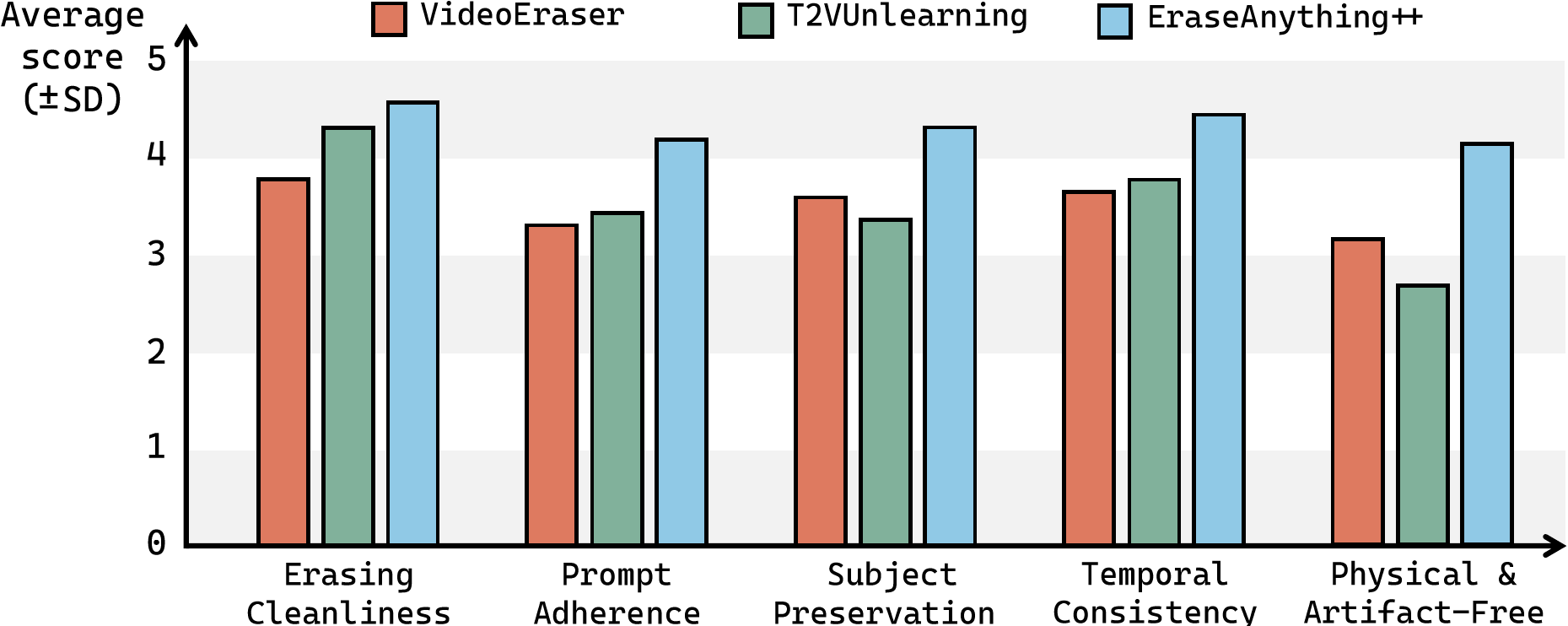}
\caption{\textbf{Human evaluation of video concept erasure on Open-Sora.}
All scores are rated from 1 to 5, where higher is better.}
\label{fig:video_human_study}
\vspace{-0.1in}
\end{figure}

As shown in Fig.~\ref{fig:video_human_study}, EraseAnything++ achieves the best human preference scores across all five dimensions. 
This suggests that our method not only removes the target concept effectively, but also better preserves the main subject, prompt alignment, temporal coherence, and visual plausibility.

\subsection{Ablation study}

To thoroughly validate the effectiveness of our proposed framework, we conduct ablation studies from three perspectives: the contribution of individual loss components in the image domain, the necessity of our spatio-temporal strategy in the video domain, and the efficacy of our multi-objective optimization mechanism.

\textbf{Ablation on loss functions.}
To assess our loss functions, we conduct an ablation study on the task of celebrity image erasure. We choose a subset from the CelebA~\cite{celeba}, omitting those that Flux.1 [dev] can't accurately reconstruct. This results in a dataset of 100 celebrities, split into two groups: 50 for erasure and 50 for retention (see Appendix C-B). 

Different variations and their results are presented in Table~\ref{tab:ablation_loss} and Fig.~\ref{fig:celeb_ablation}. Solely relying on the aggressive erasure objective ($\mathcal{L}_{esd} + \mathcal{L}_{attn}$) yields the strongest erasure efficacy with the lowest \textsc{Acc}$_{e}$, but this comes at the cost of catastrophic forgetting, significantly degrading the preservation of unrelated semantics. Conversely, introducing $\mathcal{L}_{lora}$ shifts the balance towards preservation, but failing to effectively remove the target concept, as indicated by a high \textsc{Acc}$_{e}$. The incorporation of $\mathcal{L}_{rsc}$ is crucial for distinguishing the target from synonyms, improving the trade-off. Ultimately, by combining all these loss terms, our full method achieves the highest $H_a$ score, demonstrating the most effective equilibrium between precise removal and semantic retention.

\begin{figure}[t]
\centering
\includegraphics[width=0.7\linewidth]{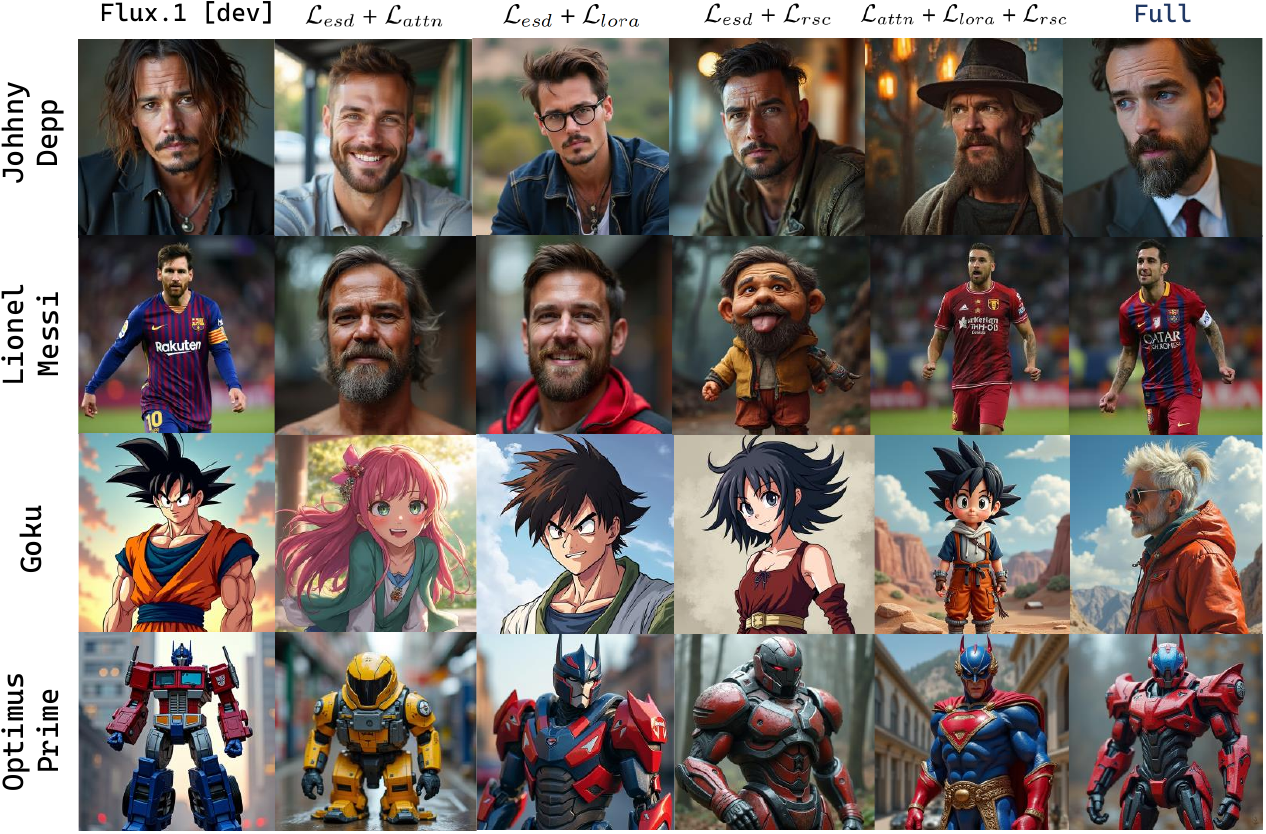}
\caption{\textbf{Ablation of different loss configs.} Our full method achieves the most effective equilibrium between concept erasure and semantic preservation.}
\label{fig:celeb_ablation}
\vspace{-0.2in}
\end{figure}

\begin{table}[h]
\caption{\textbf{Ablation on Loss Components.} We report CLIP score of Erasure Accuracy (\textsc{Acc}$_e \downarrow$), Preservation Accuracy (\textsc{Acc}$_{ir} \uparrow$), and their difference score ($H_a = \textsc{Acc}_{ir} - \textsc{Acc}_e \uparrow$). While biased configurations may excel in a single metric, our full method achieves the best overall balance ($H_a$).}
\vspace{-0.1in}
\centering
\setlength{\tabcolsep}{6pt}
\begin{tabular}{lccc}
\toprule
\textsc{Config} & \textsc{Acc}$_{e}$ $\downarrow$ & \textsc{Acc}$_{ir}$ $\uparrow$ & \textsc{H}$_{a}$ $\uparrow$ \\
\midrule
$\mathcal{L}_{esd}$ + $\mathcal{L}_{attn}$      & \textbf{19.2} & 23.5 & 4.3 \\
$\mathcal{L}_{esd}$ + $\mathcal{L}_{lora}$      & 26.8 & \textbf{28.8} & 2.0 \\
$\mathcal{L}_{esd}$ + $\mathcal{L}_{rsc}$       & 21.5 & 27.1 & 5.6 \\
$\mathcal{L}_{attn}$ + $\mathcal{L}_{rsc}$      & 23.1 & 26.5 & 3.4 \\
$\mathcal{L}_{attn}$ + $\mathcal{L}_{lora}$ + $\mathcal{L}_{rsc}$ & 21.8 & 27.9 & 6.1 \\
\midrule
\rowcolor{blue!10}
\textbf{Full (EraseAnything++)}       & 20.1 & 28.2 & \textbf{8.1} \\
\bottomrule
\end{tabular}
\label{tab:ablation_loss}
\vspace{-0.1in}
\end{table}

\textbf{Ablation on Spatio-Temporal Strategy.}
Migrating erasure to the video domain introduces unique challenges regarding temporal consistency. We qualitatively validate our \textit{Anchor-and-Propagate} strategy by comparing it with two incomplete variants on Open-Sora, as visualized in Fig.~\ref{fig:video_ablation}. When we omit the anchor frame constraint and rely solely on volumetric tuning (w/o Anchor), the target concept leaks into the initial frame and subsequently propagates throughout the video due to the lack of initial suppression. On the other hand, erasing only the first frame without applying the 3D attention constraint (w/o Propagation) leads to \textit{temporal drift}, where the erased concept hallucinates and re-emerges in later frames. Our full strategy, which couples anchor frame erasure with volumetric propagation, is the only configuration that ensures consistent and stable erasure across the entire temporal dimension.

Moreover, to assess whether our method depends on the appearance of a concept in the first frame, we analyze its performance in two scenarios: (1) when the erased concept is present from the start, and (2) when it appears only in later frames. As shown in Fig.~\ref{fig:anchor_timing}, our approach effectively erases the target regardless of its initial presence or absence, leveraging spatio-temporal propagation to prevent both persistent and delayed occurrences. This indicates that the anchor-and-propagate strategy does not solely rely on the first-frame anchor and remains robust for concepts emerging at different temporal locations.

\begin{table}[t]
\centering
\caption{
\textbf{Stratified evaluation by the temporal position of the target.}
Prompts are split by whether the target is visible in the first frame of the \textit{original} Open-Sora generation; the two strata hold $68\%/32\%$ of the Gen prompts and $78\%/22\%$ of the ImageNet prompts.
Metrics follow Tables~\ref{tab:video_nudity},~\ref{tab:video_imagenet} and~\ref{tab:concept_drift}.
}
\label{tab:anchor_strat}
\setlength{\tabcolsep}{4pt}
\resizebox{0.8\columnwidth}{!}{
\begin{tabular}{lcccc}
\toprule
\multirow{2}{*}{\textsc{Stratum}}
& \multicolumn{2}{c}{\textsc{Nudity (Gen)}}
& \multicolumn{2}{c}{\textsc{ImageNet}} \\
\cmidrule(lr){2-3} \cmidrule(lr){4-5}
& Rate $\downarrow$ & CDR $\downarrow$ & ESR-1 $\uparrow$ & CDR $\downarrow$ \\
\midrule
Target in first frame & 18.04 & 6.32 & 94.62 & 6.18 \\
Target absent from first frame & 15.70 & 9.01 & 92.48 & 8.59 \\
\midrule
\rowcolor{blue!10} All prompts & 17.29 & 7.18 & 94.15 & 6.71 \\
\bottomrule
\end{tabular}
}
\vspace{-0.15in}
\end{table}

We complement this qualitative evidence by recomputing the quantitative metrics within each stratum. Table~\ref{tab:anchor_strat} shows that erasure is effective in both: the residual nudity rate is in fact slightly lower when the concept emerges only later, and erasure success on ImageNet differs by about two points. Late-emerging targets are somewhat harder to keep suppressed, as their higher CDR indicates, but the gap is small and the method does not degenerate when the anchor frame is free of the target.

\begin{figure}[t]
\centering
\includegraphics[width=0.7\linewidth]{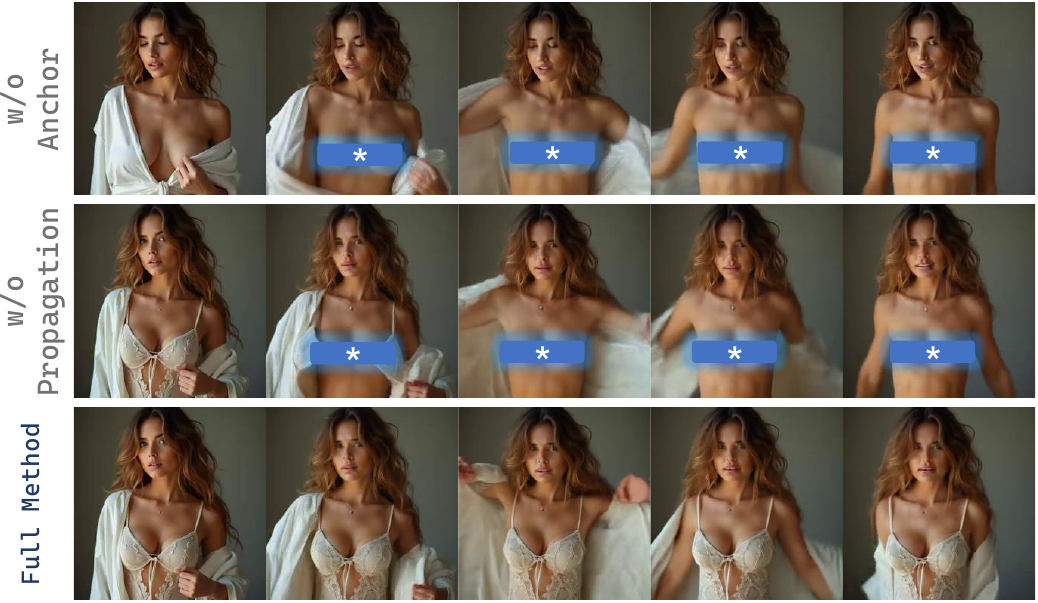}
\vspace{-0.1in}
\caption{\textbf{Ablation of the Spatio-Temporal strategy.} \textit{(Top) w/o Anchor:} The target concept is not suppressed at initialization, leading to immediate leakage. \textit{(Middle) w/o Propagation:} Although erased initially, the concept re-emerges in later frames due to temporal drift. \textit{(Bottom) Full Strategy:} Our method ensures consistent erasure across the entire video duration.}
\label{fig:video_ablation}
\vspace{-0.1in}
\end{figure}

\begin{figure}[t]
\centering
\includegraphics[width=0.7\linewidth]{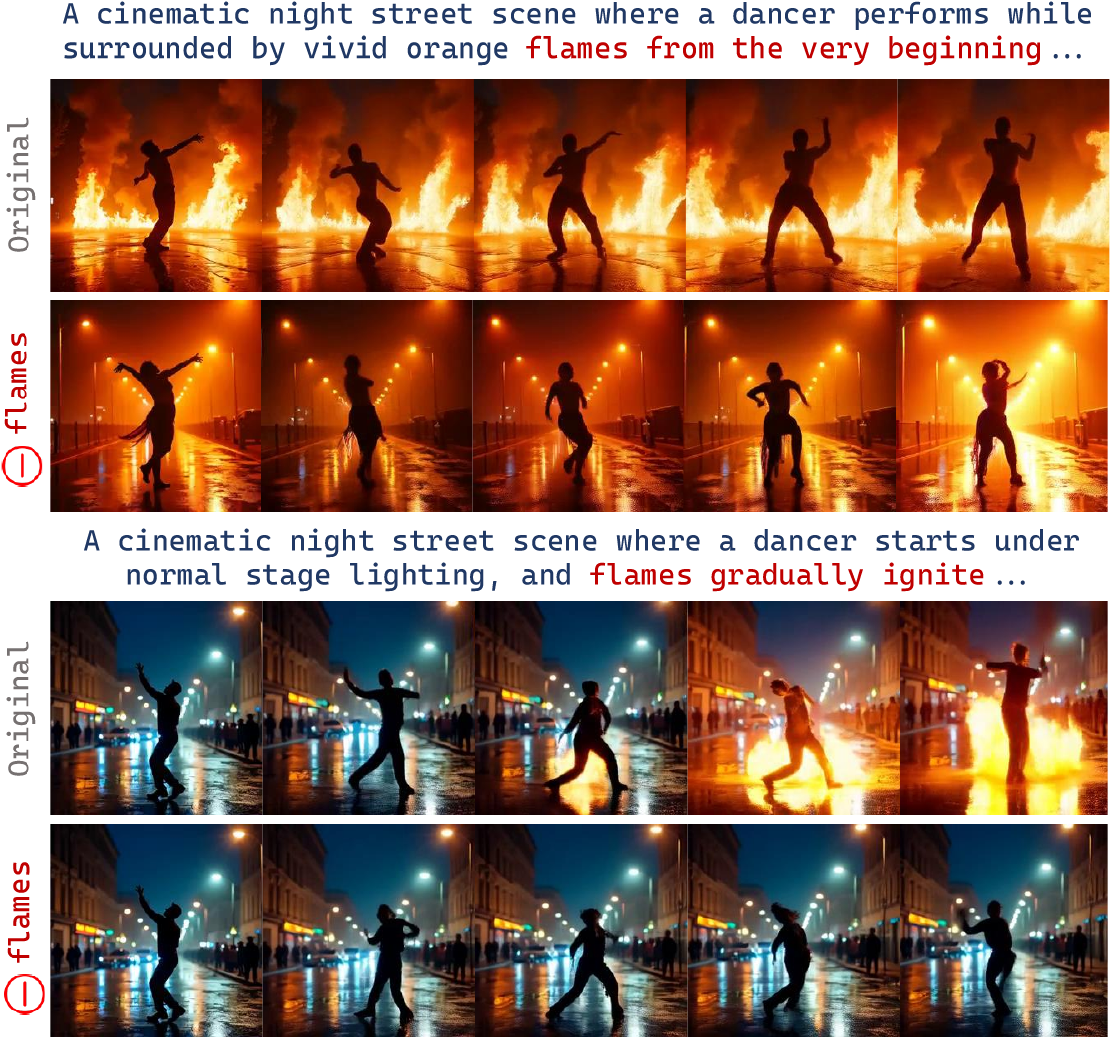}
\caption{
\textbf{Sensitivity to target appearance timing.}
\textit{(Top) Target present in the first frame:} When the erased concept already appears from the beginning, our method suppresses it consistently across the entire video.
\textit{(Bottom) Target absent in the first frame:} When the scene starts without the target concept and the concept emerges only in later frames, our spatio-temporal propagation still prevents its delayed reappearance.
These results show that the method does not rely solely on the first-frame anchor and remains effective for concepts appearing at different temporal locations.
}
\label{fig:anchor_timing}
\vspace{-0.2in}
\end{figure}

\textbf{Ablation on Optimization Strategy.} 
Finally, we justify the necessity of our proposed Multi-Objective Optimization (MOO) framework by benchmarking it against various optimization paradigms. We compare our approach with standard \textit{Linear Scalarization} ($\mathcal{L}_{total} = \mathcal{L}_e + \lambda \mathcal{L}_p$) and advanced gradient manipulation techniques, specifically \textit{PCGrad}~\cite{yu2020gradient} and \textit{MGDA}~\cite{sener2018multi}.

As detailed in Table~\ref{tab:ablation_moo}, the \textit{Linear Scalarization} strategy proves extremely sensitive to the hyperparameter $\lambda$. A small weight ($\lambda=0.1$) biases the optimization heavily towards erasure, achieving a low $\textsc{Acc}_e$ but causing catastrophic forgetting of unrelated concepts. Conversely, increasing the weight ($\lambda=10.0$) enforces rigid preservation but stifles the erasure process, rendering the unlearning ineffective.

Furthermore, advanced MOO techniques like \textit{PCGrad} and \textit{MGDA}, while designed to mitigate gradient interference, yield suboptimal results in the context of concept erasure. These algorithms treat objectives symmetrically, aiming to find a common descent direction that improves all tasks simultaneously. However, in concept erasure, the gradient for removal often directly conflicts with preservation. To satisfy the symmetric improvement condition, PCGrad and MGDA tend to adopt a conservative update strategy, dampening the erasure gradient to align with preservation. Consequently, while they maintain high $\textsc{Acc}_{ir}$, they fail to effectively erase the target concept, resulting in poor $\text{Acc}_e$. In contrast, our approach fundamentally models the problem as an asymmetric constrained optimization. By dynamically projecting the erasure gradient only when it violates the preservation constraint, our method breaks the optimization deadlock. This allows EraseAnything++ to achieve deep erasure while maintaining high specificity, yielding the superior Pareto stationary point.

\begin{table}[h]
\caption{\textbf{Ablation on Optimization Strategy.} Evaluated on CelebA (same as Table~\ref{tab:ablation_loss}). General MOO solvers (PCGrad~\cite{yu2020gradient}, MGDA~\cite{sener2018multi}) tend to be conservative, prioritizing preservation at the cost of weak erasure. Our asymmetric strategy yields the best balance.}
\vspace{-0.1in}
\centering
\setlength{\tabcolsep}{5pt}
\begin{tabular}{lccc}
\toprule
\textsc{Strategy} & \textsc{Acc}$_{e}$ $\downarrow$ & \textsc{Acc}$_{ir}$ $\uparrow$ & \textsc{H}$_{a}$ $\uparrow$ \\
\midrule
Linear Sum ($\lambda=0.1$) & \textbf{18.5} & 22.5 & 4.0 \\
Linear Sum ($\lambda=10.0$) & 28.6 & \textbf{28.9} & 0.3 \\
\midrule
PCGrad~\cite{yu2020gradient} & 24.1 & 28.5 & 4.4 \\
MGDA~\cite{sener2018multi}   & 25.8 & 28.7 & 2.9 \\
\midrule
\rowcolor{blue!10}
\textbf{Ours} & 20.1 & 28.2 & \textbf{8.1} \\
\bottomrule
\end{tabular}
\label{tab:ablation_moo}
\end{table}


\begin{table}[t]
\centering
\caption{
\textbf{Sensitivity analysis of the concept construction step on nudity erasure.}
All variants are evaluated on I2P for erasure and MS-COCO 10K for general generation preservation.
The middle block replaces only the LLM-based selection of irrelevant concepts by uniform random sampling from WordNet nouns or the T5 vocabulary, while NLTK synonym expansion is retained.
Apart from the concept construction strategy, all training settings are identical.
}
\label{tab:sensitivity_llm}
\setlength{\tabcolsep}{3pt}
\resizebox{\columnwidth}{!}{
\begin{tabular}{lccc}
\toprule
\textsc{Concept Construction} & \textsc{NudeNet Total} $\downarrow$ & FID $\downarrow$ & CLIP $\uparrow$ \\
\midrule
\rowcolor{blue!10} GPT-4o + NLTK (default)        & 182 & \textbf{21.67} & \textbf{30.35} \\
Claude-3.5-Haiku + NLTK        & \textbf{178} & 21.68 & 30.33 \\
Gemini-1.5-Flash + NLTK        & 183 & 21.81 & 30.31 \\
Qwen2.5-7B-Instruct + NLTK     & 185 & 21.89 & 30.24 \\
Qwen2.5-1.5B-Instruct + NLTK   & 190 & 22.26 & 29.97 \\
\midrule
Random WordNet nouns + NLTK    & 186 & 22.68 & 29.81 \\
Random T5 vocabulary + NLTK    & 191 & 22.95 & 29.66 \\
\midrule
w/o NLTK synonyms              & 251 & 21.72 & 30.30 \\
w/o LLM-selected irr. concepts & 181 & 23.08 & 29.54 \\
\bottomrule
\end{tabular}
}
\end{table}

\noindent \textbf{Sensitivity Analysis of the Concept Construction Step.}  Two questions arise about this step: which model performs the selection, and whether its semantics matter. For the first, Table~\ref{tab:sensitivity_llm} shows that replacing GPT-4o by another proprietary model, by an open-weight model, or by a much smaller one keeps all three metrics within a narrow band, with only a mild loss of preservation at the smallest scale. The procedure thus depends neither on a particular model nor on proprietary access. For the second, we replace the LLM selection by uniform sampling from WordNet nouns and from the T5 vocabulary, leaving synonym expansion unchanged. Erasure is essentially unaffected, since irrelevant concepts serve the preservation objective, but FID and CLIP degrade markedly and recover only a small part of the gap to the variant that uses no irrelevant concepts. What the LLM contributes is therefore not distractors as such, but concepts that remain plausible substitutes in the prompt, the very judgement T5 embeddings cannot make. The two components are complementary: removing NLTK synonyms sharply increases detected nudity, whereas removing the LLM selection mainly costs preservation.


\begin{table}[t]
\centering
\caption{
\textbf{Sensitivity analysis of the error tolerance $\varepsilon$.}
Nudity erasure is evaluated on I2P using NudeNet Total, while general generation preservation is evaluated on MS-COCO 10K using FID and CLIP.
For different concept types, we report the overall trade-off score $H_a$ following the protocol in Table~\ref{tab:table_mis}.
The default setting $\varepsilon=0.05$ provides a stable balance across different concept categories.
}
\label{tab:sensitivity_epsilon}
\setlength{\tabcolsep}{3pt}
\resizebox{\columnwidth}{!}{
\begin{tabular}{c ccc ccc}
\toprule
$\varepsilon$
& \multicolumn{3}{c}{\textsc{Nudity Erasure}}
& \multicolumn{3}{c}{$H_a$ \textsc{on Different Categories}} \\
\cmidrule(lr){2-4} \cmidrule(lr){5-7}
& \textsc{NudeNet Total} $\downarrow$
& FID $\downarrow$
& CLIP $\uparrow$
& Entity $\uparrow$
& Abstraction $\uparrow$
& Relationship $\uparrow$ \\
\midrule
0.00 & 236 & \textbf{21.43} & \textbf{30.46} & 12.8 & 4.3 & 6.2 \\
0.02 & 204 & 21.55 & 30.41 & 13.9 & 4.8 & 6.8 \\
\rowcolor{blue!10}  0.05 & 182 & 21.67 & 30.35 & \textbf{14.5} & 5.1 & \textbf{7.1} \\
0.10 & 171 & 22.04 & 30.08 & 14.3 & \textbf{5.3} & 6.9 \\
0.20 & \textbf{160} & 22.83 & 29.64 & 13.4 & 4.4 & 6.2 \\
\bottomrule
\end{tabular}
}
\end{table}

\textbf{Sensitivity to the error tolerance $\varepsilon$.}
In our work, we establish a theoretical bound on global utility degradation through the error tolerance parameter~$\varepsilon$. To empirically assess its effect and the suitability of a single value, we conduct additional experiments. In our framework, smaller~$\varepsilon$ imposes stricter preservation; larger~$\varepsilon$ enables stronger erasure but reduces generative utility. As shown in Table~\ref{tab:sensitivity_epsilon}, increasing~$\varepsilon$ improves erasure (lower NudeNet Total) yet degrades unrelated content, in line with our theoretical analysis. The default~$\varepsilon=0.05$ achieves optimal or near-optimal results for all concept types, so we use this value to ensure stable, well-balanced performance without per-concept tuning.

\begin{table}[t]
\centering
\caption{
\textbf{Sensitivity to the error tolerance $\varepsilon$ in the video setting.}
Same solver and protocol as in the image domain. Metrics follow Tables~\ref{tab:video_nudity} and~\ref{tab:concept_drift}.
}
\vspace{-0.1in}
\label{tab:sensitivity_epsilon_video}
\setlength{\tabcolsep}{5pt}
\resizebox{\columnwidth}{!}{
\begin{tabular}{c cccc}
\toprule
$\varepsilon$
& \textsc{Nudity Rate} $\downarrow$
& \textsc{Obj. Class} $\uparrow$
& \textsc{Dyn. Degree} $\uparrow$
& \textsc{CDR} $\downarrow$ \\
\midrule
0.00 & 24.86 & \textbf{90.18} & \textbf{57.02} & 11.42 \\
0.02 & 20.13 & 89.94 & 56.61 & 8.65 \\
\rowcolor{blue!10} 0.05 & 17.29 & 89.43 & 55.83 & 7.18 \\
0.10 & 15.94 & 87.66 & 53.14 & \textbf{6.94} \\
0.20 & \textbf{14.27} & 84.31 & 48.72 & 7.53 \\
\bottomrule
\end{tabular}
}
\vspace{-0.15in}
\end{table}

\textbf{Transfer of $\varepsilon$ to the video setting.}
As the same solver governs both modalities, we repeat the sweep on Open-Sora.
Table~\ref{tab:sensitivity_epsilon_video} shows the same trade-off as in images: erasure strengthens monotonically with~$\varepsilon$, while benign generation and motion degrade; CDR decreases up to $\varepsilon=0.10$ and rises only at $0.20$.
Video is the more sensitive of the two; we hypothesize that this is because the constraint acts on a spatio-temporal volume and the admitted degradation may compound across frames, which Remark~\ref{remark:utility_preservation} (a bound over optimization iterations) does not cover.
We keep the default of $0.05$, which remains a balanced choice in both settings.

\begin{table}[t]
\centering
\caption{
\textbf{Sensitivity analysis of key coefficients.}
We vary one coefficient at a time while keeping the others fixed to the default setting:
$\beta=0.1$ and $\eta=2$.
Nudity erasure is evaluated on I2P using NudeNet Total, while general generation preservation is evaluated on MS-COCO 10K using FID and CLIP.
Bold numbers denote the best value in each metric, and the shaded row indicates our default balanced setting.
}
\label{tab:coef_sensitivity}
\setlength{\tabcolsep}{3pt}
\resizebox{\columnwidth}{!}{
\begin{tabular}{c ccc @{\hspace{10pt}} c ccc}
\toprule
\multicolumn{4}{c}{\textsc{Dual Update Step-size} $\beta$}
& \multicolumn{4}{c}{\textsc{Erasure Strength} $\eta$} \\
\cmidrule(lr){1-4} \cmidrule(lr){5-8}
$\beta$
& \textsc{NudeNet Total} $\downarrow$ & FID $\downarrow$ & CLIP $\uparrow$
& $\eta$
& \textsc{NudeNet Total} $\downarrow$ & FID $\downarrow$ & CLIP $\uparrow$ \\
\midrule
0.01 & \textbf{174} & 22.05 & 30.04 & 1.0 & 231 & \textbf{21.52} & \textbf{30.43} \\
0.05 & 178 & 21.82 & 30.21 & 1.5 & 201 & 21.59 & 30.39 \\
\rowcolor{blue!10}
0.10 & 182 & 21.67 & 30.35 & 2.0 & 182 & 21.67 & 30.35 \\
0.50 & 194 & 21.58 & 30.40 & 3.0 & 158 & 22.31 & 29.86 \\
1.00 & 203 & \textbf{21.55} & \textbf{30.42} & 4.0 & \textbf{151} & 23.04 & 29.31 \\
\bottomrule
\end{tabular}
}
\end{table}

\noindent\textbf{Sensitivity to key coefficients.}
We further analyze the sensitivity of the dual update step-size $\beta$ and the erasure strength $\eta$.
As shown in Table~\ref{tab:coef_sensitivity}, increasing $\beta$ makes the dual variable respond more strongly to preservation degradation, leading to better general generation preservation but weaker erasure.
Therefore, smaller $\beta$ achieves lower NudeNet Total, while larger $\beta$ gives better FID and CLIP.
For $\eta$, a similar but more direct erasure-preservation trade-off is observed: increasing $\eta$ strengthens the negative guidance for concept removal, but overly large values degrade general generation quality due to over-suppression.
Although some settings achieve the best value on a single metric, we adopt $\beta=0.1$ and $\eta=2$ as the default balanced setting.
For $\gamma_1$ and $\gamma_2$, we observe only marginal differences after optimization convergence under moderate changes, so we keep their default values and omit detailed sensitivity plots for brevity.

\section{Conclusion}
\label{sec:conclusion}

In this paper, we propose \logopic \textbf{EraseAnything++}, a framework designed to address concept erasure in modern Flow-based Transformers for both image and video generation. By reformulating the unlearning task as a constrained Multi-Objective Optimization (MOO) problem via implicit gradient surgery, we effectively resolve the asymmetric conflict between erasing target concepts and preserving model utility, thereby mitigating the notorious risks of catastrophic forgetting and overfitting. Furthermore, we extend our approach to the temporal dimension with a Spatio-Temporal Anchor-and-Propagate strategy to counteract concept drift in videos. Extensive experiments across diverse tasks demonstrate that our method establishes a superior Pareto frontier, proving its effectiveness and versatility as a robust solution for safe generative AI.

\section{Acknowledgement}
 This work was supported by the National Natural Science Foundation of China under Grant No. 62441617. It was also supported by the Postdoctoral Fellowship Program and China Postdoctoral Science Foundation under Grant No. 2024M764093 and Grant No. BX20250485, the Beijing Natural Science Foundation under Grant No. 4254100, and by Beijing Advanced Innovation Center for Future Blockchain and Privacy Computing.

\bibliographystyle{IEEEtran}
\bibliography{egbib}


 

\vspace{-33pt}
\begin{IEEEbiography}[{\includegraphics[width=1in,height=1.25in,clip,keepaspectratio]{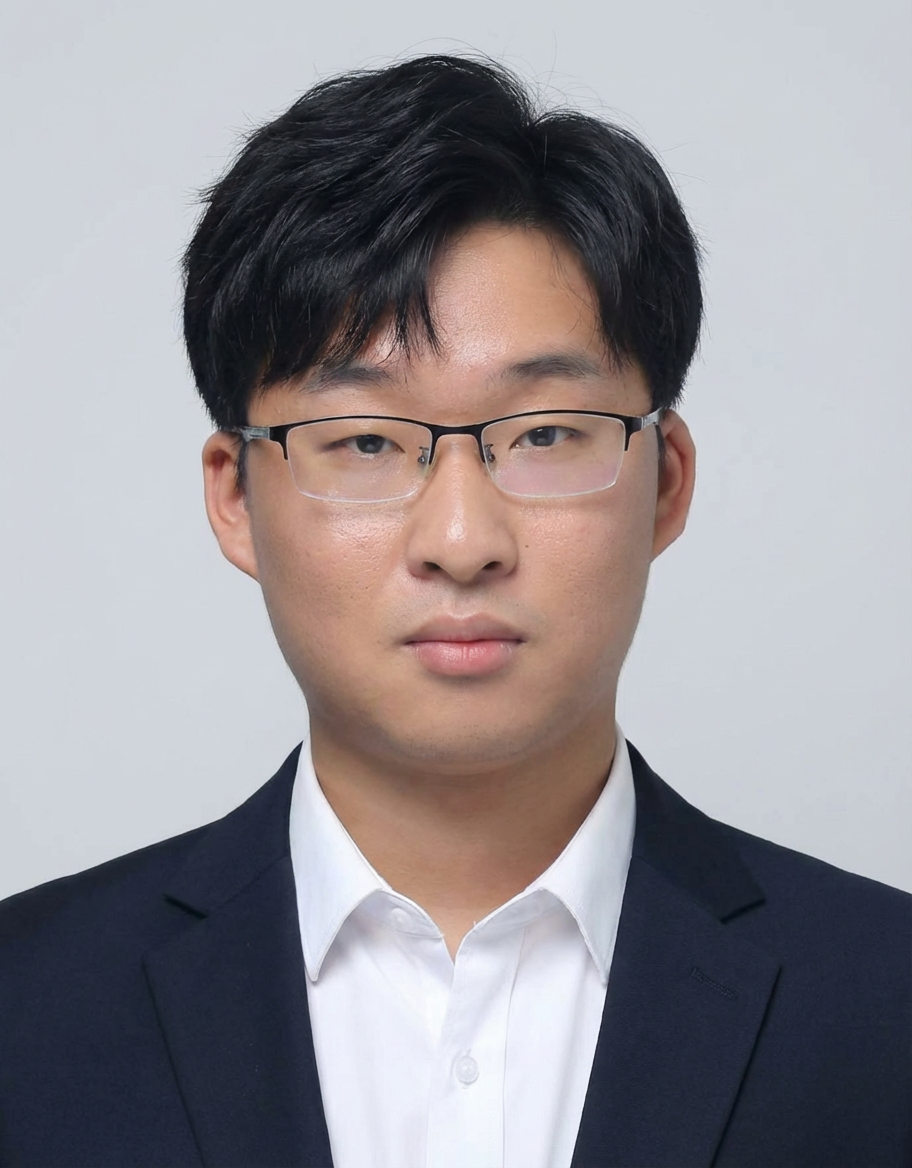}}]{Zhaoxin Fan} received his Ph.D. degree from the School of Information, Renmin University, China in 2024. He has also conducted research at Carnegie Mellon University and  Hong Kong University of Science and Technology. He is currently an Assistant Professor in the Institute of Artificial Intelligence, Beihang University. His research interests include multi-modal large language models (MLLMs), computer vision, and embodied AI.
\end{IEEEbiography}

\vspace{-33pt}
\begin{IEEEbiography}[{\includegraphics[width=1in,height=1.25in,clip,keepaspectratio]{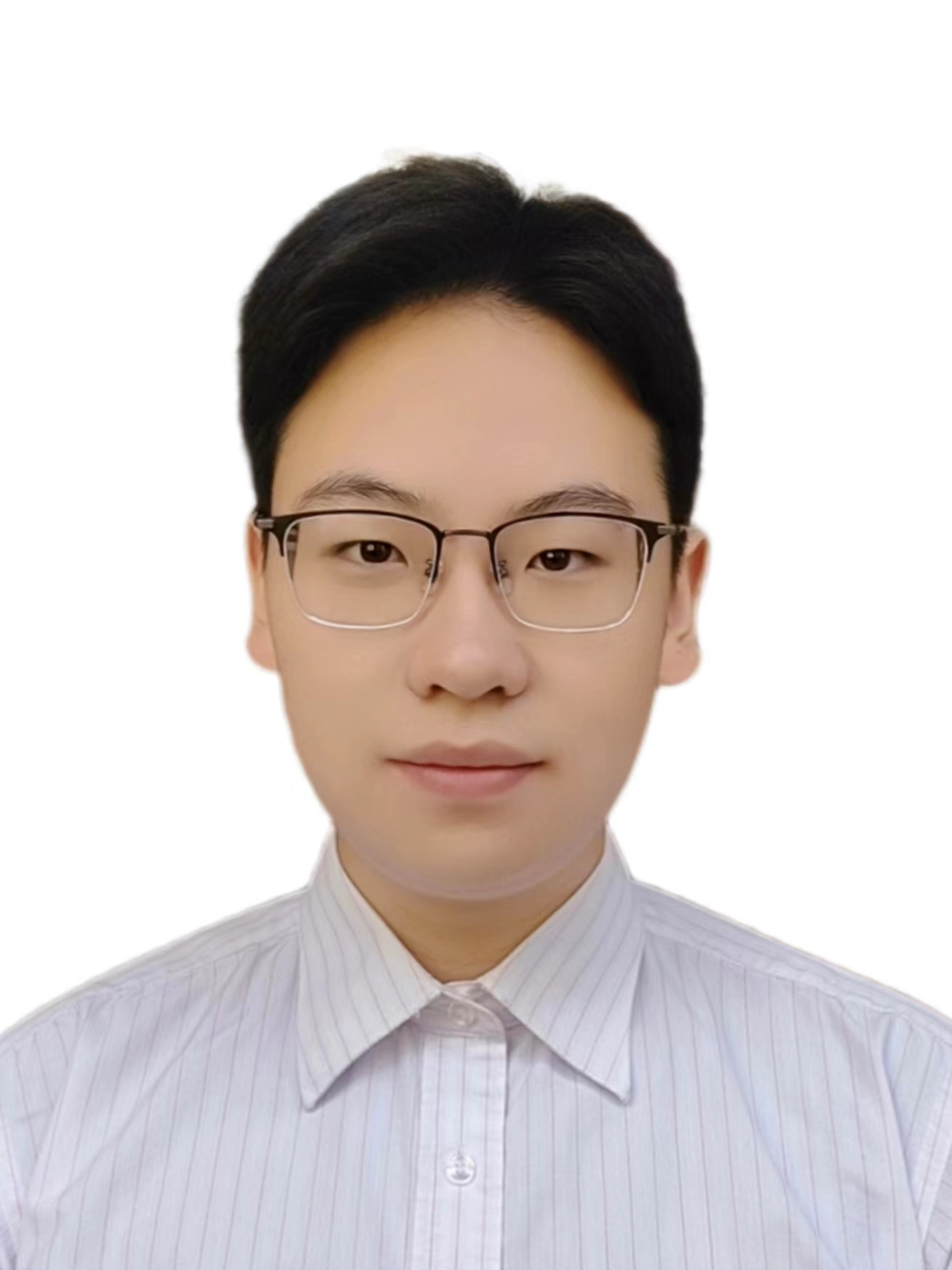}}]{Nanxiang Jiang} is currently an undergraduate student in the Institute of Artificial Intelligence, Beihang University. His research interests include computer vision, neural representations, AI for science, and embodied AI.
\end{IEEEbiography}

\vspace{-33pt}
\begin{IEEEbiography}[{\includegraphics[width=1in,height=1.25in,clip,keepaspectratio]{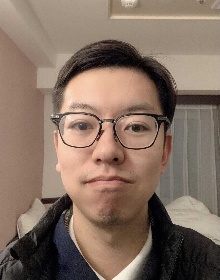}}]{Daiheng Gao} began his Ph.D. studies at the University of Science and Technology of China (USTC) in 2024. Prior to this, he served as a Senior Research Expert at Tongyi Lab, Alibaba Group. His research interests include 3D computer vision, concept erasure, and neural rendering.
\end{IEEEbiography}

\vspace{-33pt}
\begin{IEEEbiography}[{\includegraphics[width=1in,height=1.25in,clip,keepaspectratio]{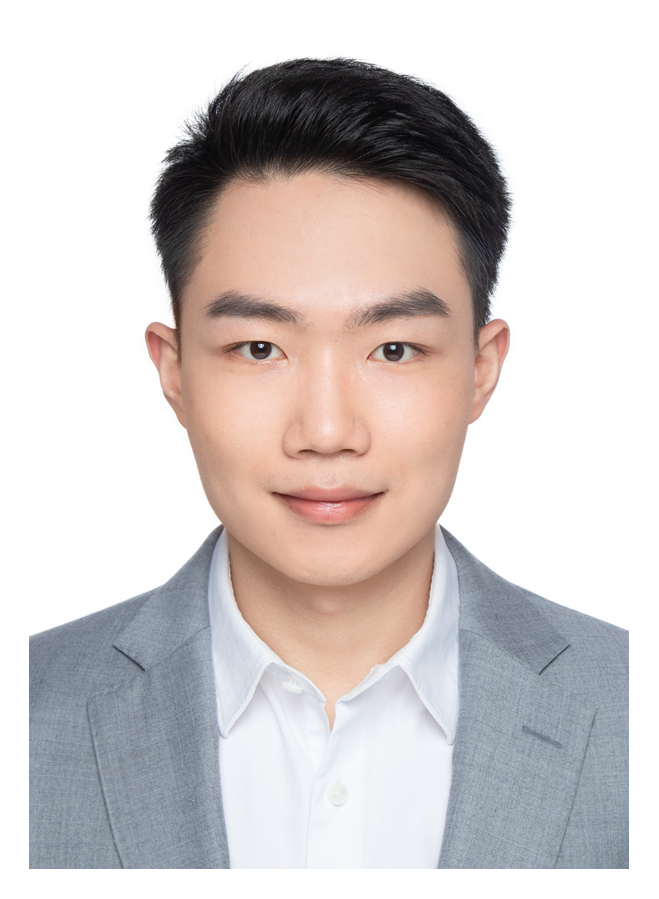}}]{Shiji Zhou} received his Ph.D. degree in computer science from Tsinghua University, in 2023. He is currently an Associate Professor in the Institute of Artificial Intelligence, Beihang University. His research interests include machine learning, optimization, and agentic RL.
\end{IEEEbiography}

\vspace{-33pt}
\begin{IEEEbiography}[{\includegraphics[width=1in,height=1.25in,clip,keepaspectratio]{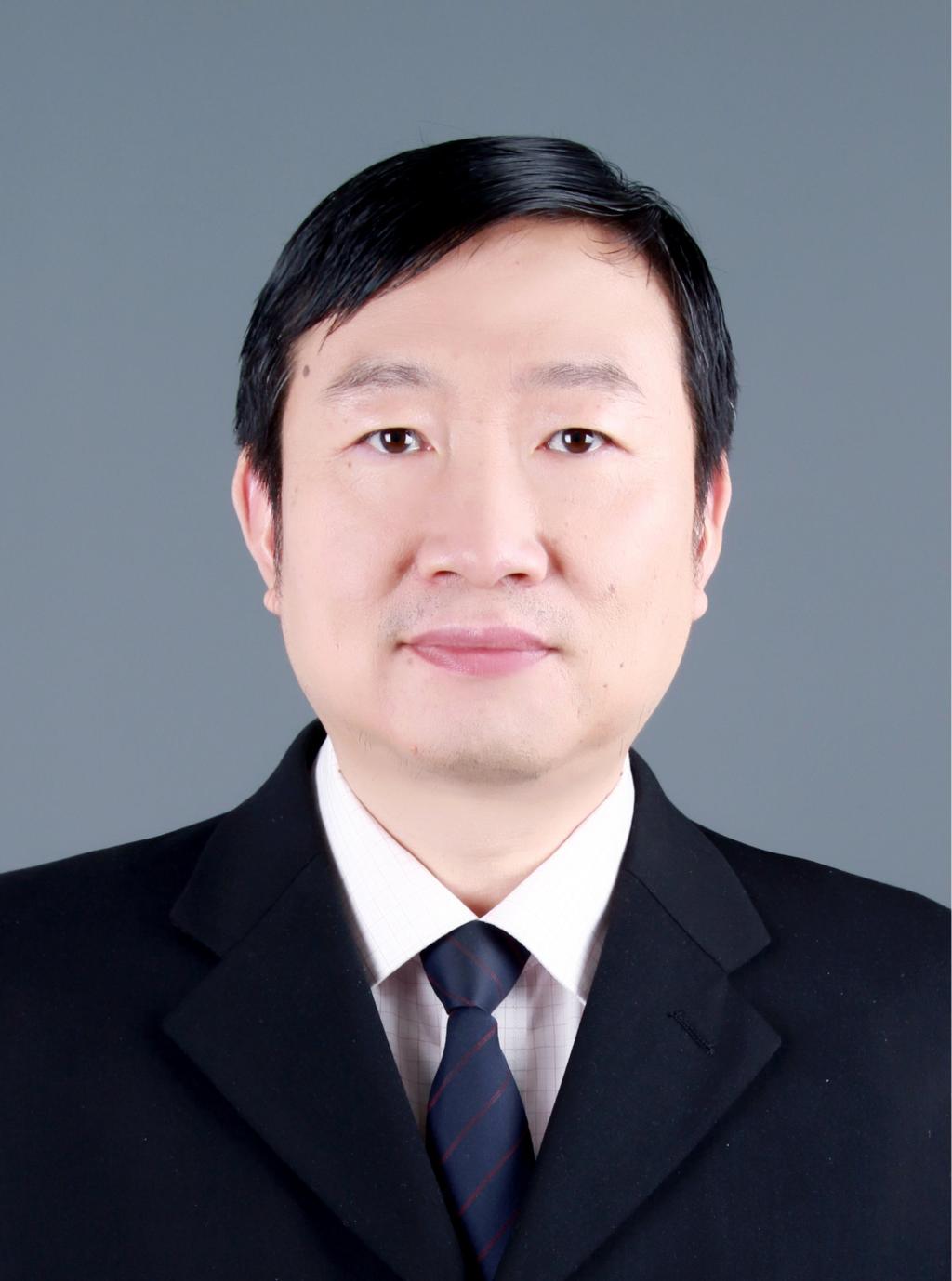}}]{Wenjun Wu} received the PhD degree in computer science from Beihang University, in 2001. He was
employed with Argonne National Laboratory as a
research scientist working on grid computing, cloud
computing, media collaboration, etc., until 2012. He
is currently employed with Beihang University as a
professor. His research interests include crowdsourcing, machine learning, cloud computing, eScience,
and cyber infrastructure.
\end{IEEEbiography}



\end{document}